\documentclass[11pt]{article}
\usepackage[margin=1in]{geometry}
\usepackage[T1]{fontenc}
\usepackage[utf8]{inputenc}
\usepackage{authblk}
\usepackage{lmodern}
\usepackage{amsmath,amssymb,amsthm,mathtools,bm}
\usepackage{booktabs,array,graphicx,xcolor}
\usepackage{enumitem}
\usepackage{algorithm,algpseudocode}
\usepackage{float}
\usepackage[round,authoryear]{natbib}
\usepackage{microtype}
\usepackage[hidelinks]{hyperref}
\usepackage{url}
\hypersetup{pdftitle={Intrinsic Associative Memory on Riemannian Manifolds: Curvature, Capacity, and Emergent Modes},pdfauthor={Anonymous Authors},pdfsubject={Riemannian associative memory}}
\newtheorem{theorem}{Theorem}
\newtheorem{proposition}[theorem]{Proposition}
\newtheorem{corollary}[theorem]{Corollary}
\newtheorem{lemma}[theorem]{Lemma}
\theoremstyle{definition}

\newtheorem{assumption}[theorem]{Assumption}
\theoremstyle{remark}

\newcommand{\M}{\mathcal M}
\newcommand{\R}{\mathbb R}
\newcommand{\E}{\mathbb E}
\renewcommand{\P}{\mathbb P}
\newcommand{\ind}{\mathbf 1}
\newcommand{\geo}{\mathrm{geo}}
\newcommand{\vc}{\mathrm{vc}}
\newcommand{\emg}{\mathrm{em}}
\newcommand{\tot}{\mathrm{tot}}
\DeclareMathOperator{\grad}{grad}
\DeclareMathOperator{\Hess}{Hess}
\DeclareMathOperator{\Ric}{Ric}
\DeclareMathOperator{\Scal}{Scal}
\DeclareMathOperator{\vol}{vol}
\DeclareMathOperator{\inj}{inj}
\DeclareMathOperator{\Var}{Var}
\DeclareMathOperator{\Poisson}{Poisson}

\newcommand{\norm}[1]{\left\lVert #1\right\rVert}
\newcommand{\ip}[2]{\left\langle #1,#2\right\rangle}

\title{\bf Intrinsic Associative Memory on Riemannian Manifolds:\\Curvature, Capacity, and Emergent Modes}
\author[1]{Krishnakumar Balasubramanian}
\author[2]{Zhaoyang Shi}

\affil[1]{Department of Statistics, University of California, Davis. \texttt{kbala@ucdavis.edu}}
\affil[2]{Center for Applied Mathematics, Fudan University. \texttt{zyshi@fudan.edu.cn}}

\date{}
\begin{document}
\maketitle
\begin{abstract}
Geometry does more than constrain an associative memory: curvature
determines what it remembers and which states it creates. We develop
intrinsic dense associative memories on Riemannian manifolds by casting
memory as Epanechnikov kernel-density mode seeking. We compare geodesic and
volume-corrected energies and show that curvature separates their behavior.
We prove that geodesic memory always retains an isolated pattern, while
corrected memory obeys a sharp Ricci-curvature threshold:
positive curvature can erase memories in high dimensions, while negative
curvature reinforces them. We derive geodesic capacity scalings of
$q_\beta^{-1/2}$ for retaining every pattern and $q_\beta^{-1}$ for a
typical one, where $q_\beta$ is the pairwise kernel-overlap probability. We
show how overlap \emph{creates} novel memories: designed $N$-pattern
configurations realize all $2^N-1$ subset modes, but random data at the
storage threshold yield only a Poisson number. We establish exact one-step
recall using Riemannian mean shift. In simulations, we recover the predicted
curvature transition and every designed mode. On WordNet's full noun
hierarchy, we demonstrate that volume correction improves low-capacity retrieval. Together, our work shows that curvature is a design variable for associative memory, not merely a property of the data.
\end{abstract}

\section{Introduction}

Modern AI increasingly relies on manifold-valued representations:
hyperbolic embeddings encode lexical hierarchies such as WordNet
\citep{nickel2017poincare,wang2026curved}, normalized neural embeddings lie
on hyperspheres \citep{wangisola}, and orientations, covariance matrices,
and shapes inhabit $SO(3)$, positive-definite manifolds, and Kendall shape
spaces, respectively \citep{chatterjee2013,barachant2012,kendall1984}.
The same geometric considerations arise in non-Euclidean foundation
models \citep{yang2025noneuclidean,he2025beyondeuclidean} and generative
modeling on manifolds \citep{debortoli2022,chen2024flow}.
In these settings, geometry is part of the semantics, with direct
consequences for associative memory: Euclidean updates can produce
invalid states---averaged rotations need not be rotations, and perturbed
covariance matrices need not remain positive definite---while projection
back onto the manifold can alter distances and hence retrieval dynamics.
These limitations call for associative memories that operate intrinsically
on manifolds, extending beyond the predominantly binary or Euclidean
formulations of existing high-capacity and modern continuous models
\citep{krotov,demircigil,krotov2021,ramsauer}.

The central challenge in this extension is accounting for curvature.
Tangent spaces provide a first-order Euclidean approximation, but curvature
alters geodesic distances, local volume, and intrinsic averaging beyond
this approximation. These effects directly shape the memory's energy
landscape: changes in the energy Hessian affect local stability, changes
in support overlap influence interference and capacity, and shifts in the
balance among memory contributions can create or remove additional modes.
Consequently, constructions that coincide in Euclidean space can yield
different attractors and retrieval dynamics under positive or negative
curvature. We therefore ask how curvature governs storage, retrieval,
and the creation of new memory states.

\begin{figure}[!t]
\centering
\includegraphics[width=\linewidth]{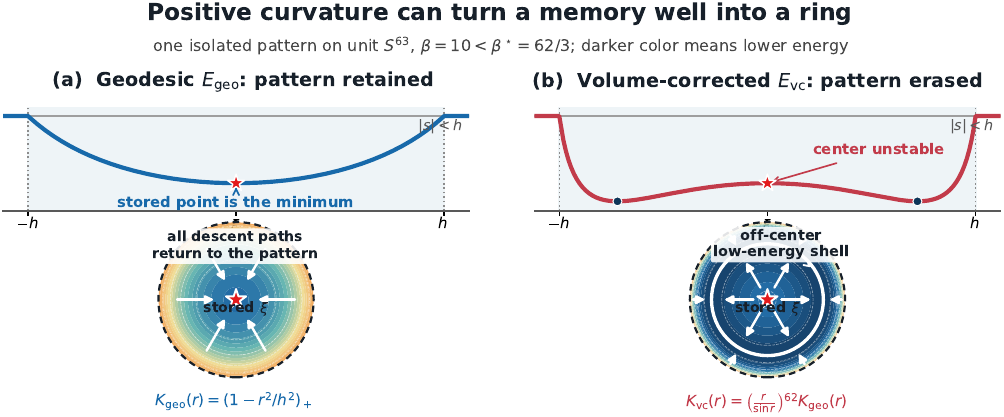}
\caption{Curvature separates two energies that coincide in flat space.
Radial profiles and normal-coordinate support disks show one isolated
pattern on the unit $S^{63}$ at $\beta=10<\beta^\star=62/3$ and
$\epsilon=1$; darker colors denote lower energy and arrows indicate
descent. The geodesic energy retains its minimum at $\xi$. Multiplication
by $\theta_\xi^{-1}=(r/\sin r)^{62}$ instead makes $\xi$ unstable for the
volume-corrected energy and creates an off-center low-energy shell. The
shell visualizes erasure of the original, not an additional nondegenerate
memory; above threshold the center is restored, while negative curvature
reinforces it. See Section~\ref{sec:model} for the precise definitions of
the two energies.\vspace{-0.2in}}
\label{fig:intro-energy}
\end{figure}

We study these questions using Epanechnikov dense associative memory,
which represents memory states as modes of a compactly supported
kernel-density score \citep{hoover}. Compact support isolates distant
patterns, while overlapping supports can generate additional modes.
On a Riemannian manifold, we compare two formulations: a geodesic-distance
energy that replaces Euclidean distance with geodesic distance, and a
volume-corrected KDE energy that additionally compensates for the
exponential map's volume distortion \citep{pelletier}. The two coincide
in flat (Euclidean) space, but under curvature the volume correction can reshape
the landscape and destabilize memory modes. We analyze both to determine
how a correction motivated by density estimation affects associative
memory; Figure~\ref{fig:intro-energy} previews this distinction.

For retrieval, we adopt the Riemannian mean-shift update associated with
the geodesic score \citep{subbarao}. It averages logarithmic displacements
to active patterns and maps the result back through the exponential map,
recalling a single active pattern in one step and taking an intrinsic
averaging step when several patterns are active. We treat stored and
newly created states uniformly as nondegenerate modes. Our analysis
combines curvature expansions to characterize isolated-pattern stability,
close-pair geometry and probability to quantify interference and capacity,
overlap geometry and geodesic convexity to characterize the formation
and uniqueness of additional modes, and descent estimates to establish
retrieval guarantees. Throughout, we distinguish simultaneous retention,
typical-pattern retention, and prescribed retrieval basins. Newly created
modes constitute geometric novelty, which does not by itself imply
statistical or semantic generalization.

Our contributions in this work are as follows:
\begin{itemize}[leftmargin=*,noitemsep]
\item \textbf{Intrinsic models, curvature, and retrieval guarantees.}
We formulate the geodesic and volume-corrected energies and compute their
exact local Hessians. An isolated pattern is always stable for the geodesic
model, whereas corrected stability obeys a sharp Ricci threshold; positive
curvature can destroy storage and negative curvature can strengthen it.
We also connect the geodesic energy to Riemannian mean shift and prove its
gradient relation, exact one-step recall in singleton neighborhoods, descent
even when the active set changes, and local linear convergence near a
memory (Section~\ref{sec:model}, Theorems~\ref{prop:descent}
and~\ref{thm:isolated}, and Figure~\ref{fig:main-simulation}).

\item \textbf{Capacity under geometric interference.}
For independent patterns in the stated regimes, we identify retention with
the absence of kernel-scale collisions and derive distinct laws for
retaining every pattern and retaining a typical pattern. We also give
high-dimensional sphere rates and expose the bandwidth tradeoff between
capacity and guaranteed retrieval neighborhoods
(Section~\ref{sec:capacity}, Theorem~\ref{thm:capacity}, and
Proposition~\ref{prop:typical}). Writing $q_\beta$ for the pairwise
support-overlap probability, the corresponding scales are
$q_\beta^{-1/2}$ and $q_\beta^{-1}$.

\item \textbf{A theory of emergent memories.}
We characterize each mode by its active subset and show that it is an
intrinsic mean or a curvature-corrected balance point. Designed
configurations realize all $2^N-1$ subset modes, whereas random data at the
storage threshold produce only a conditional Poisson number. This separates
what is geometrically possible by design from what appears typically in
random data
(Section~\ref{sec:emergence},
Theorems~\ref{thm:active}--\ref{thm:emergence}, and
Figure~\ref{fig:main-simulation}).

\item \textbf{Experimental tests and scope.}
Controlled experiments recover the predicted curvature transition and
every designed mode. On the full WordNet noun hierarchy, volume correction
yields a reproducible retrieval gain in $16$-memory banks across three
levels of partial evidence and five independently trained embeddings.
Together, these experiments test curvature-controlled stability, designed
emergence, and semantic retrieval on manifold-valued representations
(Section~\ref{sec:experiments},
Figures~\ref{fig:main-simulation}--\ref{fig:main-wordnet}, and
Appendix~\ref{app:wordnet}).
\end{itemize}

A detailed discussion of related work and the position of our contribution
within associative memory, geometric statistics, and manifold learning is
provided in Appendix~\ref{app:related-work}.

\section{Curvature-controlled intrinsic memory and retrieval}\label{sec:model}
We study two intrinsic Epanechnikov energies, connecting corrected manifold
KDE \citep{pelletier} to the geodesic score used in Riemannian mean shift
\citep{subbarao}. Our analysis supplies the exact Ricci criterion, corrected
balance equation, quantitative emergence construction, random retention and
emergence laws, and memory guarantees for the established mean-shift update.
All proofs are in the appendix.

\paragraph{Geometry in three operations.}
Let $(\M,g)$ be a smooth, connected, compact $m$-dimensional Riemannian manifold without boundary, and let $\xi_1,\ldots,\xi_N$ be distinct patterns. The distance $d(x,p)$ is the length of a shortest path on $\M$. The logarithm $\log_xp\in T_x\M$ is its initial displacement, of length $d(x,p)$; the exponential $\exp_xv$ moves from $x$ along displacement $v$. We use a kernel radius $h<\inj(\M)$, so these paths are unique inside each support. Put
\[
 \beta=2/h^2,\qquad q_i(x)=\tfrac12d(x,\xi_i)^2,
 \qquad a_i(x)=\theta_{\xi_i}(x)^{-1}.
\]
Here $\theta_p$ is the exponential map's volume density:
$dV(\exp_pv)=\theta_p(\exp_pv)\,dv$. It measures how Euclidean volume in the tangent space is distorted on the manifold; it equals one in flat space.

\paragraph{Model.}
For $\epsilon\ge0$, define the single-pattern kernels, scores, and energies
\begin{align}
 K_\geo(x,\xi_i)&=(1-\beta q_i(x))_+,
 &K_\vc(x,\xi_i)&=a_i(x)(1-\beta q_i(x))_+,\label{eq:kernels}\\
 S_\geo(x)&=\epsilon+\sum_i K_\geo(x,\xi_i),
 &E_\geo(x)&=-\beta^{-1}\log S_\geo(x),\label{eq:geo}\\
 S_\vc(x)&=\epsilon+\sum_i K_\vc(x,\xi_i),
 &E_\vc(x)&=-\beta^{-1}\log S_\vc(x).\label{eq:vc}
\end{align}
Both kernels vanish at distance $h$. The corrected kernel is defined as zero
outside its support, without evaluating $a_i$ there. If $\epsilon=0$, we
consider only $S_T>0$. The negative logarithm is decreasing, so energy
minima are score maxima.

\paragraph{What is a memory?}\label{def:memory}
A memory is a point where $E_T$ is $C^2$ in a neighborhood, $\grad E_T=0$, and $\Hess E_T\succ0$, for $T\in\{\geo,\vc\}$. The Hessian condition gives a locally attracting minimum under sufficiently small gradient steps. An $\eta$-\emph{novel} memory is at distance at least $\eta>0$ from every original. Write $N_\tot^T$ for all memories and $N_\emg^T(\eta)$ for novel ones. \emph{Global emergence} requires both a novel memory and retention of every original. Thus memories are nondegenerate modes, excluding plateaus and degenerate families.

\subsection{Riemannian mean-shift retrieval}\label{sec:layer}
At a query $x$, define the active set
$A(x)=\{i:d(x,\xi_i)<h\}$. Algorithm~\ref{alg:radius} is Riemannian
mean shift with the flat weights induced by the Epanechnikov profile. For
$\kappa(s)=(1-s)_+$, $-\kappa'(s)=\ind\{s<1\}$ away from the boundary:
this is the classical flat/Epanechnikov shadow-kernel relation
\citep{cheng,comaniciu}. Substituting these weights into
\citet{subbarao}, equations~(29)--(30), gives
\begin{equation}
 \alpha_i(x)=\frac{\ind\{i\in A(x)\}}{|A(x)|},\qquad
 v(x)=\sum_{i\in A(x)}\alpha_i(x)\log_x\xi_i,\qquad
 T_t(x)=\exp_x\!\bigl(t v(x)\bigr).
 \label{eq:sparseattention}
\end{equation}
The usual mean-shift step has $t=1$; $t$ controls its length. An empty
neighborhood returns the query with an inactive flag. In attention
language, the rule selects keys by radius and weights them uniformly, not
by kernel height. This is the mean-shift direction for the geodesic score.
The corrected energy uses the same support gate but requires an additional
volume-gradient term, given in Appendix~\ref{app:descent} and reflected in
the corrected balance equation of Section~\ref{sec:structure}.

\begin{algorithm}[H]
\caption{Riemannian mean-shift retrieval with flat weights}\label{alg:radius}
\begin{algorithmic}[1]
\Require Query $x\in\M$, keys $\xi_1,\ldots,\xi_N$, radius $h<\inj(\M)$, step $t>0$
\State $A\gets\{i:d(x,\xi_i)<h\}$
\If{$A=\varnothing$}
 \State \Return $(x,\mathrm{inactive})$
\EndIf
\State $v\gets |A|^{-1}\sum_{i\in A}\log_x\xi_i$
\State \Return $(\exp_x(tv),\mathrm{active})$
\end{algorithmic}
\end{algorithm}

\begin{theorem}[Memory guarantees for Riemannian mean shift]\label{prop:descent}
At a query away from support boundaries with $k=|A(x)|>0$,
\begin{equation}
 \grad S_\geo(x)=\beta k v(x),\qquad \grad E_\geo(x)=-\frac{k}{S_\geo(x)}v(x),
 \label{eq:gradientlayer}
\end{equation}
so $T_t$ is a gradient step with the query-computable step size
$tS_\geo(x)/k$. If $A(x)=\{i\}$, then $T_1(x)=\xi_i$ exactly. In
particular, write $s=\min_{i\ne j}d(\xi_i,\xi_j)$, with $s=\infty$ when
$N=1$. Every query in $B_\Delta(\xi_i)$ retrieves $\xi_i$ in one step
whenever
\begin{equation}
 0<\Delta<h,\qquad \Delta\le s-h.
 \label{eq:retrievalradius}
\end{equation}
More generally, suppose each $q_i$, $i\in A(x)$, is smooth and has
$\Hess q_i\preceq\Lambda g$ along
$\gamma(u)=\exp_x(uv(x))$, $0\le u\le t$. If $0<t<2/\Lambda$, then
\begin{equation}
 S_\geo(T_t(x))-S_\geo(x)
 \ge\beta k t(1-\Lambda t/2)\norm{v(x)}^2.
 \label{eq:descentmain}
\end{equation}
Consequently the energy decreases strictly unless $v(x)=0$, even if the
active set changes. On a fixed manifold, sufficiently small $h$ permits the
full step $t=1$. Near a memory with fixed active set and
$cg\preceq\Hess q_i\preceq\Lambda g$, admissible fixed steps converge
locally linearly, with derivative contraction at most
$\max\{|1-tc|,|1-t\Lambda|\}<1$.
\end{theorem}

\paragraph{From mode seeking to recall.}
The score-dependent step size requires no knowledge of the target identity.
A singleton neighborhood yields exact recall; with multiple keys, the
update takes one step toward their intrinsic mean without computing it
exactly. The descent argument follows classical mean-shift monotonicity
\citep{comaniciu}; the Hessian condition provides an explicit admissible
step size on the manifold, including across changes in the active set.

Each iteration computes $N$ distances, $|A|$ logarithmic maps, and one
exponential map, with manifold-dependent costs. The update is isometry
equivariant (Appendix~\ref{app:descent}). Outside the full-step regime,
the Hessian bound or energy-based backtracking provides a descent step.
The theorem establishes local convergence near a memory, without
guaranteeing convergence from arbitrary initializations.

\paragraph{Values, sparsity, and learning.}
The keys are manifold-valued patterns. Euclidean values $u_i\in\R^d$
admit the readout $r(x)=\sum_i\alpha_i(x)u_i$;
Theorem~\ref{prop:descent} applies to the query update $T_t(x)$, not to
arbitrary value maps. In the similarity--separation--projection taxonomy
of \citet{millidge}, negative squared distance defines similarity, the
radius gate imposes hard separation, and the exponential map returns
the tangent aggregate to the manifold.

Unlike graded sparsemax/entmax weights \citep{hu,santos}, the flat weights
can be discontinuous at support boundaries. For a fixed active set and
fixed $t$, the update depends smoothly on the query and keys, but its
derivative with respect to $h$ vanishes. Radius selection therefore
requires a separate tuning rule; smoothing the gate modifies the
exact-retrieval model. The corrected energy additionally requires the
Jacobian-gradient term in~\eqref{eq:correctedmean}, as established in
Appendix~\ref{app:descent}.

\subsection{Curvature-controlled stability}
\begin{theorem}[Curvature changes stability, not stationarity]\label{thm:isolated}
This pointwise statement also holds on a complete noncompact manifold with $h<\inj(\M)$. If $d(\xi_i,\xi_j)>h$ for all $j\ne i$, both gradients vanish at $\xi_i$, and
\begin{equation}
 \Hess E_\geo(\xi_i)=\frac{g_{\xi_i}}{1+\epsilon},\qquad
 \Hess E_\vc(\xi_i)=\frac{g_{\xi_i}-\Ric_{\xi_i}/(3\beta)}{1+\epsilon}.
 \label{eq:ricci}
\end{equation}
Thus the geodesic model always stores this isolated pattern. The corrected model stores it if and only if
\begin{equation}
 \beta>\tfrac13\lambda_{\max}(\Ric_{\xi_i}),
 \label{eq:threshold}
\end{equation}
where the eigenvalue is relative to $g_{\xi_i}$. On the unit sphere $S^m$, this is $\beta>(m-1)/3$, equivalently $h^2<6/(m-1)$ for $m>1$.
\end{theorem}

\paragraph{Why curvature appears.}
Ricci curvature governs the leading local volume distortion through
$\theta_p(\exp_p v)=1-\Ric_p(v,v)/6+O(\norm{v}^3)$.
The vanishing linear term preserves stationarity at the kernel center,
whereas the quadratic term modifies its stability. Below the stability
threshold, a spherical corrected kernel attains its maximum on a shell
of positive radius rather than at its center
(Appendix~\ref{app:retrieval}). On hyperbolic space $\mathbb H^m$ of
curvature $-1$, $\Ric=-(m-1)g$, so the Hessian of the isolated corrected
energy is $(1+(m-1)/(3\beta))g/(1+\epsilon)$; negative curvature therefore
strengthens local stability. At equality in~\eqref{eq:threshold}, the
Hessian is degenerate. The criterion is exact for every admissible $h$
and requires none of the additional counting assumptions introduced below.

\section{Storage capacity: all patterns, a typical pattern, and bandwidth}\label{sec:capacity}
Let $X_1,\ldots,X_N$\footnote{Here and in the corresponding appendix
proofs, $X_i$ denotes a random stored pattern. For a realization $\omega$,
writing $\xi_i=X_i(\omega)$ recovers the fixed-pattern notation of
Section~\ref{sec:model}.} be independent with common law $\nu$ admitting a
density relative to $dV$. Write $\mathcal A_T$ for retention of every
original, and $N_\delta^T(h)$ for the largest $N$ with
$\P(\mathcal A_T)\ge1-\delta$. Define
\begin{equation}
 b_h(x)=\nu(B_h(x)),\qquad q_h=\P\{d(X_1,X_2)<h\}=\int b_h(x)\,d\nu(x).
 \label{eq:collisiondef}
\end{equation}
The close-pair probability $q_h$ measures interference. For the geodesic model, an original with an active neighbor is almost surely nonstationary: exact cancellation of random displacements is a null event. The same holds for the corrected model at sufficiently small $h$ on a fixed compact manifold, ensuring both Ricci stability and nonsingularity of the neighbor-to-gradient map. Hence, in these regimes,
\begin{equation}
 \mathcal A_T=\{d(X_i,X_j)>h\text{ for every }i\ne j\}\quad\text{almost surely}.
 \label{eq:collisionevent}
\end{equation}

\paragraph{All-pattern versus typical-pattern retention.}
Under $q_h\to0$ and $\sup_xb_h(x)\le Cq_h$, with $C$ independent of $h$ on a fixed manifold or of $m$ along a sphere sequence,
\begin{equation}
 N^2q_h\to2\lambda\quad\Longrightarrow\quad
 \P(\mathcal A_T)\to e^{-\lambda},\qquad
 N_\delta^T(h)\sim\sqrt{\frac{-2\log(1-\delta)}{q_h}}.
 \label{eq:capacitymain}
\end{equation}
This is a geometric birthday problem: there are $\binom N2$ potential collisions. The close-pair Poisson tools are classical \citep{silverman,penrose}; the energy-specific step is~\eqref{eq:collisionevent}.

For a uniformly selected original, let $p_N^T(h)$ be its retention probability and $N_{\mathrm{typ},\delta}^T(h)$ the largest $N$ with $p_N^T(h)\ge1-\delta$. In the same regimes,
\begin{equation}
 p_N^T(h)=\int(1-b_h(x))^{N-1}\,d\nu(x),\qquad
 N_{\mathrm{typ},\delta}^T(h)\sim\frac{-\log(1-\delta)}{q_h}\quad(b_h\equiv q_h).
 \label{eq:typicalmain}
\end{equation}
The last formula assumes homogeneous ball probabilities, as on a uniformly sampled sphere; then $p_N^T=(1-q_h)^{N-1}$. This is also the expected retained fraction. A vanishing fraction fails in probability when $Nq_h\to0$, without requiring $N^2q_h\to0$. Proposition~\ref{prop:typical} proves these statements and gives the nonuniform-density constant. Retention alone is not a guarantee for corrupted queries.

\paragraph{Growing dimension: uniform patterns on $S^m$.}
For $m\ge2$, the collision probability and fixed-$h$ capacity rates are
\begin{equation}
 q_m(h)=\frac{\int_0^h\sin^{m-1}t\,dt}{\int_0^\pi\sin^{m-1}t\,dt},
 \label{eq:spherecap}
\end{equation}
\begin{equation}
 \log N_\delta^\geo(m,h)=-\frac m2\log(\sin h)+o(m),\qquad
 \log N_{\mathrm{typ},\delta}^\geo(m,h)=-m\log(\sin h)+o(m),
 \label{eq:spheredichotomy}
\end{equation}
for fixed $0<h<\pi/2$ and $0<\delta<1$. The typical-pattern exponent is twice the all-pattern exponent. For \emph{either} notion, corrected capacity is zero when $\beta\le(m-1)/3$ and equals the geodesic capacity above that threshold. Nonisolated originals are almost surely nonstationary, while isolated ones obey Theorem~\ref{thm:isolated}.

Choosing $h=c/\sqrt m$ with $0<c<\sqrt6$ restores corrected storage. Both models satisfy
\begin{equation}
 \log N_\delta^T(m,h)=\frac m4\log(m/c^2)+O(\log m),
 \label{eq:shrinksphere}
\end{equation}
and typical capacity has twice this leading term (Corollary~\ref{cor:sphere}). The guaranteed single-kernel retrieval radius is less than $h=O(m^{-1/2})$. Appendix~\ref{app:packing} contrasts these random capacities with designed storage at a prescribed basin radius.

\paragraph{Storage versus density consistency.}
For independent samples from a $C^4$ density on a fixed compact manifold,
retaining all patterns with high probability requires $N^2h^m\to0$
as $h\to0$, whereas mean-square density consistency requires
$Nh^m\to\infty$ (Appendix~\ref{app:normalization}). Additional stable
modes constitute geometric memories; assessing their statistical or
semantic utility requires a separate criterion.

\section{Emergent memories: balance points and counts}
\label{sec:emergence}\label{sec:structure}
\subsection{Self-consistent balance points}
The compact support makes it possible to characterize memories one subset
at a time. Recall the active set $A(x)=\{i:d(x,\xi_i)<h\}$. For a nonempty
subset $A$, let
\begin{equation}
 U_A=\bigcap_{i\in A}B_h(\xi_i),\qquad
 C_A=\{x\in U_A:d(x,\xi_j)>h\text{ for every }j\notin A\}.
 \label{eq:activeregions}
\end{equation}
The first set is the common support of the selected kernels. The second checks that \emph{only} those kernels are active. A candidate computed from $A$ is a memory only if it passes this second test.

\paragraph{The local convexity regime.}
Throughout this section, we assume that each support ball is strongly
geodesically convex: any two of its points are joined by a unique
minimizing geodesic contained in the ball. We also require
$\Hess q_i\succeq cg$ on each support for some $c>0$ and, for the corrected
model, $\Hess[a_i(1-\beta q_i)]\preceq-\mu g$ there for some $\mu>0$.
These inequalities ensure that the active score components are strongly
concave along geodesics within their supports. On a fixed smooth compact
manifold, these conditions hold for all sufficiently small $h$.
Assumption~\ref{ass:convex} provides an explicit sufficient bound in terms
of derivatives of $a_i$, and Lemma~\ref{lem:small} establishes
$\mu\ge c_0h^{-2}$ in this regime. An injectivity-radius bound ensures
smooth squared distances but does not imply these convexity conditions.

\begin{theorem}[One candidate per active subset]\label{thm:active}
Under the preceding conditions, a subset $A$ has at most one critical candidate in $U_A$ for each energy. The geodesic candidate satisfies
\begin{equation}
 \sum_{i\in A}\log_x\xi_i=0.
 \label{eq:mean}
\end{equation}
The corrected candidate satisfies
\begin{equation}
 \sum_{i\in A}a_i(x)\log_x\xi_i
 =\frac1\beta\sum_{i\in A}a_i(x)(1-\beta q_i(x))\grad\log\theta_{\xi_i}(x).
 \label{eq:correctedmean}
\end{equation}
Such a candidate is a memory with active set $A$ exactly when it lies in $C_A$. Every memory is obtained this way; none lies on a support boundary. Every critical point in the smooth supported domain is a memory, so there are no saddles there. In particular,
\begin{equation}
 N_\tot^T\le2^N-1,\qquad
 N_\emg^T(\eta)\le2^N-N-1\quad\text{if all originals are retained}.
 \label{eq:subsetbound}
\end{equation}
\end{theorem}

\paragraph{Intrinsic and curvature-adjusted means.}
Equation~\eqref{eq:mean} expresses a balance of logarithmic displacements
to the active patterns. Under the convexity assumptions, the candidate
minimizes the sum of squared geodesic distances to these patterns and is
therefore an intrinsic, or Fr\'echet, mean.
Equation~\eqref{eq:correctedmean} includes an additional force induced by
the spatially varying volume correction, so its solution is generally
not an ordinary Fr\'echet mean. In flat geometry, $a_i=1$ and this force
vanishes, recovering the Euclidean subset-mean mechanism \citep{hoover}.

The active-set condition is essential. On the Euclidean line with $h=1$,
for example, the patterns $-0.99,0.8,0.9$ are all active at the origin.
Their mean, approximately $0.237$, lies more than one unit from the first
pattern and therefore does not define a three-pattern memory.
Moreover, a candidate formed from several patterns may coincide with
a stored pattern; novelty requires a distance of at least $\eta$ from
every stored pattern.

A pair illustrates the coexistence of storage and emergence. At separation
$r\in(h,2h)$, the two patterns do not activate one another, yet their
supports overlap around the midpoint. If no third kernel interferes and
the convexity conditions hold, the midpoint is a memory of the geodesic
model, while both original patterns remain stored. For sufficiently small
bandwidth, the corrected pair admits a nearby balance point. This pair
mechanism governs the emergence count for random patterns.

Theorem~\ref{thm:active} yields a finite enumeration procedure: solve the
balance equation for each nonempty subset and retain only candidates
satisfying the active-set and novelty conditions. Although exponential
in $N$, this procedure gives an exact count for small configurations.
Appendix~\ref{app:active} provides a finite-sample occupancy bound, while
Appendix~\ref{app:spherecount} sharpens the count on spheres of fixed
dimension.

\subsection{Designed versus random emergence}
The subset upper bound does not describe a typical dataset. The next result compares a configuration designed to realize every subset with independent samples at the all-pattern storage threshold. Both statements apply to both energies.

\begin{theorem}[Exponential constructions and a Poisson random count]\label{thm:emergence}
\textbf{Designed patterns.}\label{thm:simplex}
Fix $p\in\M$ and $0<\tau<1/\sqrt2$. There is a constant $c_\tau>0$, depending on the local geometry and $\tau$, such that for every $2\le N\le m+1$ and $0<h\le c_\tau/\sqrt N$, a configuration in an $O(h)$ neighborhood of $p$ satisfies
\begin{equation}
 N_\tot^T=2^N-1,\qquad
 N_\emg^T(\tau h)=2^N-N-1,\qquad T\in\{\geo,\vc\},
 \label{eq:maximal}
\end{equation}
with every original retained. For each subset, the memory differs by $O(h^3)$ from the exponential image of its tangent-space centroid.

\textbf{Independent patterns on a fixed manifold.}
Let the $X_i$ have a continuous density $f$, let $h\to0$, and suppose $N^2q_h\to2\lambda\in(0,\infty)$. For fixed $0<\tau<1$, put $a_\tau=\max\{1,2\tau\}$ and $\lambda_\tau=(2^m-a_\tau^m)\lambda$. Then
\begin{equation}
 N_\emg^T(\tau h)\mid\mathcal A_T
 \ \Longrightarrow\ \Poisson(\lambda_\tau),
 \label{eq:emergencepoisson}
\end{equation}
where $\Longrightarrow$ denotes convergence in distribution. In particular,
\begin{equation}
 \P\{\mathcal A_T,\ N_\emg^T(\tau h)>0\}
 \longrightarrow e^{-\lambda}(1-e^{-\lambda_\tau}).
 \label{eq:globalprob}
\end{equation}
\end{theorem}

\paragraph{Persistence of the designed count under curvature.}
Construct a regular simplex in $T_p\M$ and map its vertices to $\M$
through $\exp_p$. For a subset of size $k$, its Euclidean centroid has
squared distance $h^2(1-1/k)$ to member vertices and $h^2(1+1/k)$ to
nonmembers. Thus every nonempty subset satisfies the active-set conditions
with a squared-distance margin of at least $h^2/N$. After rescaling
coordinates by $h$, curvature perturbs squared distances and kernels
by $O(h^2)$. Comparing these perturbations with the rescaled margin
$1/N$ yields the sufficient condition $h\le c_\tau/\sqrt N$.
Appendix~\ref{app:simplex} bounds the critical-point displacement and
verifies the support conditions uniformly over all subsets.

For $N=3$ on a surface, the construction yields seven memories:
three original memories, three pair memories, and one triple memory.
For $N=m+1$, the total count is $2^{m+1}-1$.
The constant $c_\tau$ may depend on dimension through the underlying
geometry. On unit spheres, Corollary~\ref{cor:quantitative-sphere}
provides dimension-independent constants in the sufficient bounds
$h\le c_\tau/\sqrt N$ for the geodesic model and
$h\le c_\tau/\sqrt{mN}$ for the corrected model.
When $N=m+1$, both sufficient radii decrease only polynomially with
dimension. The result for general manifolds does not assume uniform
geometry across dimensions.

\paragraph{Why the random count does not grow with $N$.}
In the regime $N^2q_h\asymp1$, form a graph connecting patterns within
distance $2h$. With probability tending to one, every connected component
is an isolated vertex or an isolated pair. Retaining all original
patterns excludes pairs separated by less than $h$. Each remaining pair
at separation $r$ generates a midpoint memory in the geodesic model and
a memory displaced from the midpoint by $O(h^3)$ in the corrected model.
Novelty at scale $\tau h$ additionally requires $r\ge2\tau h$ to leading
order. The contributing pair separations therefore lie in
$(a_\tau h,2h)$, yielding the volume factor $2^m-a_\tau^m$.

For $\tau\le1/2$, the limiting Poisson distribution has mean
$(2^m-1)\lambda$. At fixed dimension $m$, this mean is independent of
$N$, and the count conditional on retaining all original patterns
remains bounded in probability. This complements the exponential count
for designed configurations. The underlying graph approximation is
classical; Appendix~\ref{app:emergence} establishes its correspondence
with the exact memory count, accounting for corrected pair locations
and the conditioning event.

\paragraph{What novelty measures.}
We scale the novelty threshold with $h$: every supported memory lies
at distance less than $h$ from an original pattern, so none can satisfy
a threshold $\eta\ge h$. Density consistency concerns a distinct
statistical regime, as discussed in Section~\ref{sec:capacity}.

\section{Retrieval and generation experiments}\label{sec:experiments}
We organize the experiments around the paper's two direct memory claims:
retrieval of retained patterns and generation of additional stable modes.
Figure~\ref{fig:main-simulation} tests both claims in controlled simulations;
Figure~\ref{fig:main-wordnet} tests retrieval on the full WordNet noun
hierarchy. Detailed protocols, additional diagnostics, denoising
experiments, raw-output descriptions, and full-size figures are in
Appendices~\ref{app:additional-simulations}--\ref{app:wordnet}. The
auxiliary EEG covariance retrieval experiments are retained in
Appendix~\ref{app:eeg}.

\paragraph{Synthetic retrieval.}
We first isolate the pointwise curvature claim of
Theorem~\ref{thm:isolated}. Each of $16$ independently sampled centers is
placed in its own singleton bank, and eight fixed tangent directions are
reused throughout a bandwidth sweep. Queries begin at distance $0.3h$, with
$h=\sqrt{2/\beta}$, and recall requires terminal distance at most $0.05h$.
On $S^{63}$, $E_\geo$ has $100\%$ recall throughout, whereas $E_\vc$ switches
from zero recall through $\beta=20.5$ to full recall from $\beta=20.8$,
bracketing the exact prediction $\beta^*=62/3=20.67$. Its sub-threshold
terminal radius matches the predicted off-center shell within
$3\times10^{-4}h$. Both energies remain stable on $SO(3)$, whose threshold
lies below the admissible range, and on affine-invariant $\mathrm{SPD}(3)$,
whose Ricci curvature is nonpositive. A singleton experiment formed from
measured leaf contours exhibits the corresponding predicted transition on
Kendall shape space rather than on simulated tangent perturbations
(Appendix~\ref{app:leaf-transition}).

\paragraph{Synthetic novel generation.}
For the maximal construction in Theorem~\ref{thm:emergence}, we map a
centered tangent simplex to $S^2$ and $SO(3)$ and enumerate every nonempty
active subset. Both energies realize exactly all $2^3-1=7$ and
$2^4-1=15$ predicted stable modes at every tested bandwidth, retaining all
originals. Blind support-annulus queries discover all $4$ and $11$ emergent
modes and reach one in $43.2\%$ and $33.3\%$ of trials. In the less
structured random-bank experiment, additional modes coexist with exact
original recall in $10/12$ retained $S^2$ banks and all $12$ retained
$SO(3)$ banks. Their median held-out density percentiles are $66$ and $70$,
showing that geometric novelty need not imply low probability under the
data-generating distribution.

\begin{figure}[!ht]
\centering
\includegraphics[width=0.92\linewidth]{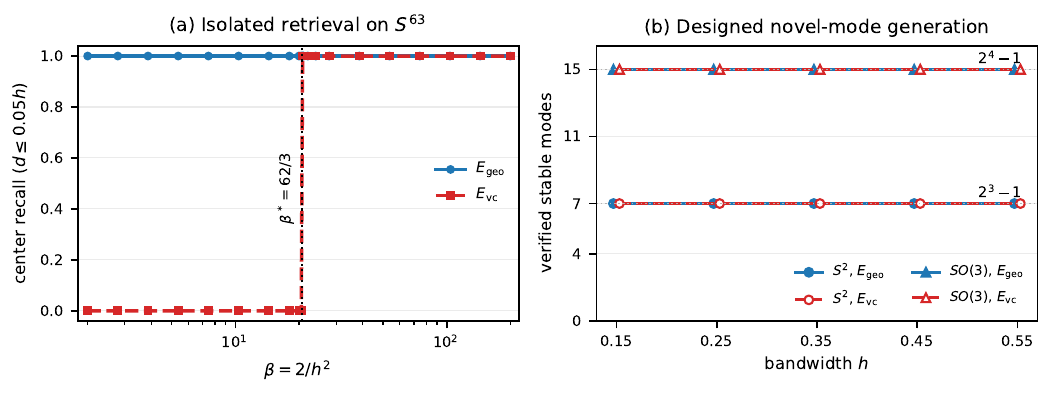}
\caption{Direct synthetic tests. Left: on $S^{63}$, the $E_\vc$ transition
brackets the exact Ricci threshold while $E_\geo$ recalls throughout
($16\times8=128$ paired trials per point). Right: both energies attain every
mode predicted by the designed simplex across the bandwidth sweep.\vspace{-0.15in}}
\label{fig:main-simulation}
\end{figure}

\paragraph{Full WordNet hierarchy.}
Topic-dependent curvature in language-model semantic spaces and gains from
curvature-aware retrieval routing motivate our text-domain test
\citep{wang2026curved}. We use all $82{,}115$ WordNet 3.0 noun synsets and
their $743{,}241$ hypernym-closure relations \citep{miller1995wordnet}.
Following \citet{nickel2017poincare}, each run trains a $10$-dimensional
Lorentz hierarchy embedding for $200$ epochs, holding out $128$ sufficiently
deep leaf concepts. A clean memory key uses all of a leaf's ancestor
relations; its query uses a nested $25\%$, $50\%$, or $75\%$ subset.
Euclidean memory on the corresponding Poincare coordinates is the
same-representation control.

The compact-support radius is selected from memory-to-memory distances to
retain at least $80\%$ of originals, without query outcomes. To average over
bank composition, each of five independent embedding and target-split seeds
uses $256$ deterministic random banks drawn from its $128$ held-out leaves.
Confidence intervals bootstrap these five seed-level means, not the
overlapping banks. Exact recovery additionally requires that the target is
retained and that the terminal point lies within $0.05h$ of it.

\paragraph{Corrected retrieval in 16-memory banks.}
With $16$ memories, $E_\vc$ exact recall is $0.419$, $0.617$, and $0.757$
as evidence increases, versus $0.339$, $0.560$, and $0.717$ for ambient
Euclidean memory. The paired gains are $0.080$, $0.057$, and $0.041$, with
$95\%$ seed-bootstrap intervals $[0.055,0.101]$, $[0.036,0.078]$, and
$[0.028,0.054]$; every seed improves. Gains over $E_\geo$ are $0.103$,
$0.093$, and $0.072$, with all three seed-bootstrap intervals bounded away
from zero. Thus corrected dynamics gives a reproducible benefit in this
low-capacity semantic retrieval setting.

\begin{figure}[!ht]
\centering
\includegraphics[width=0.92\linewidth]{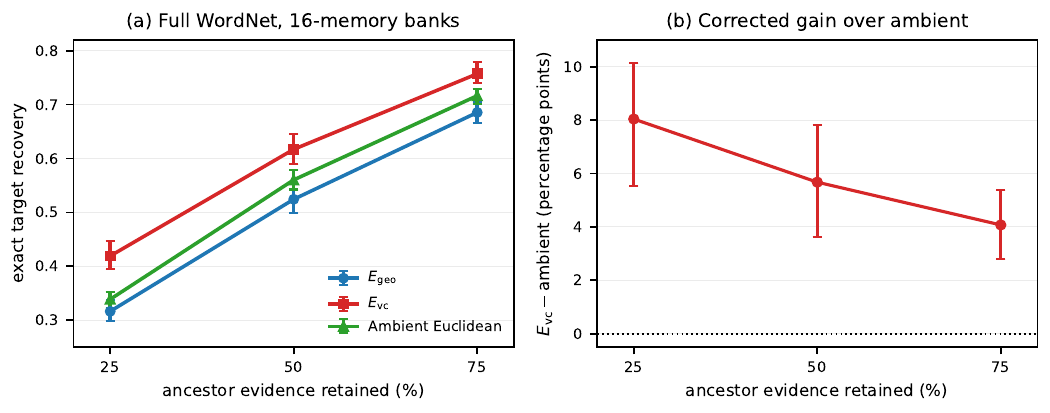}
\caption{Full-WordNet partial-relation retrieval. Left: exact recovery in
$16$-memory banks, averaged over $256$ banks for each of five embedding
seeds; bars show seed-bootstrap $95\%$ intervals. Right: paired $E_\vc$
gains over Euclidean memory on the same Poincare coordinates; all intervals
are positive. \vspace{-0.15in}}
\label{fig:main-wordnet}
\end{figure}

\section{Summary and limitations} \vspace{-0.1in}
We developed intrinsic Epanechnikov DAMs on Riemannian manifolds with
geodesic and volume-corrected energies. Curvature separates them: geodesic
memory retains isolated patterns, whereas corrected stability follows a
Ricci threshold, allowing positive curvature to erase and negative
curvature to strengthen memories. Riemannian mean shift gives exact
singleton retrieval. Support collisions yield $q_h^{-1/2}$ all-pattern and
$q_h^{-1}$ typical-pattern capacity, while overlap creates balance-point
memories: designed configurations realize all $2^N-1$ subset modes, but
random threshold-scale banks yield only Poisson-many. Experiments recover
the transition and designed modes; full-WordNet retrieval shows a volume-correction gain. Limitations include fixed-dimensional general-manifold random limits, convexity-dependent counting and
geometry-dependent constants. For future work building (looped) transformer-like layers from this DAM energy, to iterate intrinsic mean shift while learning manifold-valued keys,
queries, bandwidths, and representations end-to-end is interesting.
 \clearpage
\section*{AI use statement}
In this work, we used Generative AI to formulate mathematical claims, provide critical ingredients for proving mathematical claims (e.g., carry out some calculation that were checked by the authors), assist in the writing of proofs, design or provide feedback on research  methodology or experiments (e.g., asking AI whether a specific theoretical claim can be improved toward a specific direction), implement methods (e.g., write code), literature search (e.g., ask if we missed any existing literature), language editing, proof auditing, and \LaTeX{} assistance.  The authors independently checked the mathematical arguments, citations, code, and experimental claims and take responsibility for the final manuscript.

\section*{Acknowledgements}
KB is supported in part by National Science Foundation (NSF) grant DMS-2413426. The authors thank Benjamin Hoover, Dmitry Krotov and Parikshit Ram for helpful discussions.

\appendix
\numberwithin{theorem}{section}
\numberwithin{equation}{section}
\section{Extended related work and positioning}\label{app:related-work}

We position this work across dense associative memory, mode-seeking,
Riemannian statistics, and geometric learning by its state space, success
event, and use of geometry.

\paragraph{Associative memory and compact support.}
Dense associative memories use nonlinear interactions for storage beyond
pairwise Hopfield energies
\citep{krotov,demircigil,krotov2021}. Modern Hopfield networks connect this
view to attention through global softmax interactions
\citep{ramsauer}; the universal Hopfield framework separates similarity,
separation, and projection \citep{millidge}, and sparse or structured models
modify these choices \citep{hu,santos,stanhope}. A distributional extension
stores Gaussian measures under the $2$-Wasserstein metric and studies
barycentric fixed points, capacity, and perturbation recovery
\citep{tankala2026gaussian}. Epanechnikov memory instead uses compact support
\citep{hoover}: distant patterns cease to interact, while overlaps can create
additional modes. We transfer it to Riemannian manifolds, compare geodesic
and volume-corrected scores, derive the Ricci stability threshold, and
quantify capacity and emergence. Capacity comparisons must
match success events: \citet{demircigil} distinguish fixed-pattern stability
from simultaneous binary error correction, whereas our all-pattern and
typical-pattern capacities concern exact retention as nondegenerate manifold
modes; prescribed retrieval basins are a further requirement.

\paragraph{Mean shift and retrieval.}
Mean shift originated in density-gradient mode seeking \citep{fukunaga},
with the flat/Epanechnikov shadow-kernel relation explaining the uniform
weights over active samples \citep{cheng,comaniciu}. Riemannian mean shift
already uses logarithmic and exponential maps and distinguishes geodesic
from volume-corrected objectives \citep{subbarao}. We therefore do not claim
a new generic algorithm. Instead, we identify its flat-weight update as the
geodesic-memory retrieval rule and prove exact singleton recall, score
increase across active-set changes, and local linear convergence.
Euclidean gradient-line consistency \citep{ariascastro,ariaserrata} is a
distinct asymptotic statement and does not determine finite-sample
retention, emergence, or these retrieval guarantees.

\paragraph{Manifold KDE and geometric probability.}
The second energy uses the exponential-map volume correction from manifold
KDE \citep{pelletier}; related work studies estimator bias, geometric
structure, embedded distances, and boundaries
\citep{henry,kimpark,berry}. Our question is different: does each
finite-sample kernel center remain a stable memory? The Ricci term in the
local volume expansion shows that a statistically natural correction can
destabilize an isolated mode on positive curvature. All-pattern retention
and mean-square density consistency also require incompatible bandwidth
regimes (Appendix~\ref{app:normalization}). Intrinsic-mean convexity
\citep{karcher,afsari2011,afsari} supports the active-subset analysis,
whereas close-pair Poisson approximations \citep{silverman,penrose} underpin
the random capacity and emergence laws. Appendix~\ref{app:packing}
separates these random laws from designed storage with prescribed basins.

\paragraph{Non-Euclidean foundation models and generation.}
Hyperbolic attention \citep{gulcehre} and hierarchical embeddings
\citep{nickel2017poincare} established geometry-aware learning. Recent
position and workshop papers argue that foundation-model data exhibit
hierarchies, cycles, symmetries, multi-way relations, and non-isotropic
scaling that flat embeddings distort, motivating hyperbolic, spherical, and
mixed-curvature architectures
\citep{yang2025noneuclidean,he2025beyondeuclidean}. Riemannian score models
\citep{debortoli2022} and flow matching \citep{chen2024flow} generate valid
manifold states. These works motivate manifold-valued representations but do
not formulate a dense associative memory whose observations are finite-sample
modes or analyze curvature-dependent storage, retrieval, capacity, and
emergence. Our WordNet study uses a learned hierarchy as its retrieval state
space; it proposes neither an embedding nor a foundation model.

Our work complements modern Hopfield architectures, Riemannian mean shift, manifold
KDE, and non-Euclidean learning by explaining intrinsic compact-support
memory: curvature controls stability and balance, collisions control
capacity, and overlap creates modes. 
 
 \section{Additional retrieval and denoising experiments}
\label{app:additional-simulations}

\subsection{Detailed isolated-retrieval sweep}

To isolate the pointwise claim from pattern interference, each of $16$
independently sampled centers is placed in its own singleton bank. For each
center, eight fixed random tangent directions are reused throughout the
bandwidth sweep. Queries begin at distance $0.3h$, where
$h=\sqrt{2/\beta}$, and recall is declared when the terminal distance is at
most $0.05h$. Thus every trial satisfies the isolation hypothesis of
Theorem~\ref{thm:isolated}; no hyperparameter is selected.

Figure~\ref{fig:sim-detail} shows $100\%$ $E_\geo$ recall throughout. On
$S^{63}$, $E_\vc$ recall is zero through $\beta=20.5$ and one from
$\beta=20.8$, bracketing $\beta^*=(m-1)/3=20.67$. Below threshold, its
median terminal radius agrees with the analytically predicted off-center
shell to within $3\times10^{-4}h$. Both energies retain every center across
the tested range on $SO(3)$ and affine-invariant $\mathrm{SPD}(3)$.

\begin{figure}[!ht]
\centering
\includegraphics[width=\linewidth]{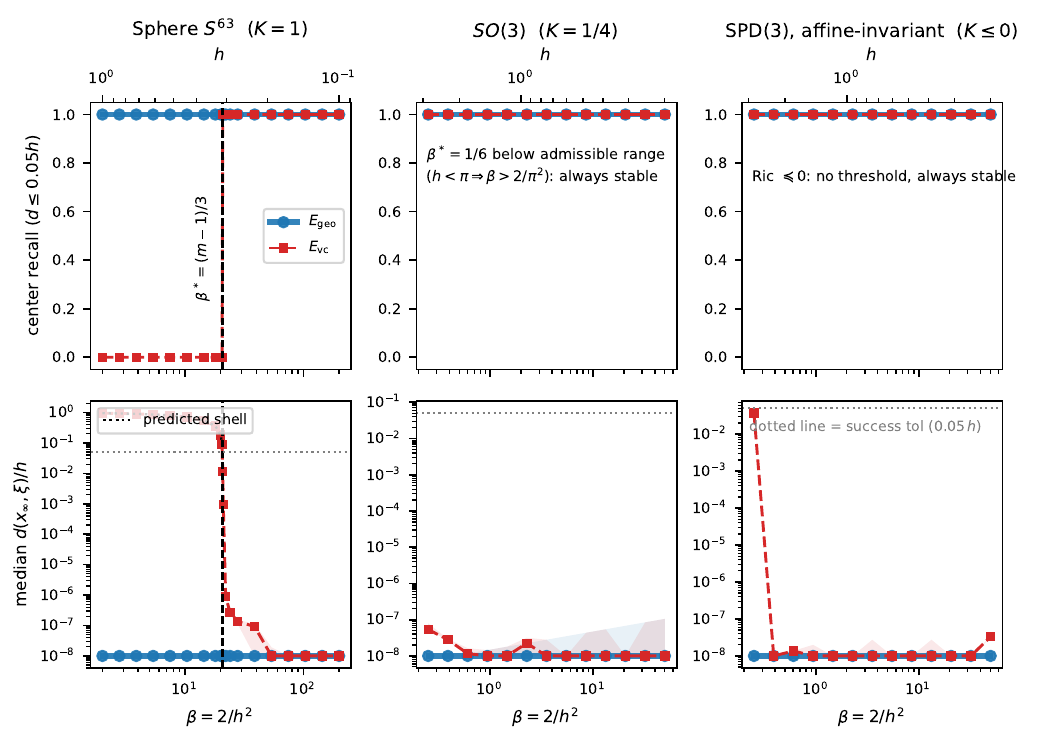}
\caption{Isolated singleton retrieval versus $\beta=2/h^2$. Top: recall
within $0.05h$. Bottom: median terminal distance normalized by $h$, with
repetition quartiles. The $E_\vc$ spherical model switches at
$\beta^*=62/3$ and its sub-threshold endpoint tracks the analytic shell.
On $SO(3)$ and $\mathrm{SPD}(3)$ the threshold does not bind.}
\label{fig:sim-detail}
\end{figure}

\subsection{Real-leaf replication of the isolated transition}
\label{app:leaf-transition}

We next test whether the pointwise transition remains visible when both the
memory centers and query perturbations are obtained from measured shapes.
We use all $225$ scans from three classes of the Swedish Leaf dataset
(Alnus incana, Salix alba `Sericea', and Salix cinerea), with $75$ scans per
class \citep{soderkvist2001}. Each segmented contour is represented by
$k=16$ petiole/apex-anchored landmarks. Removing translation and scale and
quotienting global planar rotation gives Kendall planar shape space
$\mathbb{CP}^{k-2}=\mathbb{CP}^{14}$ with Fubini--Study distance
\[
 d([z],[w])=\arccos |z^*w|.
\]
In this normalization, $\Ric=2(14+1)g=30g$, so
Theorem~\ref{thm:isolated} predicts the strict corrected-memory threshold
$\beta^*=30/3=10$.

For each scan, the contour extracted with maximum image side $512$ pixels
is placed in its own singleton bank. The paired query is produced by
resizing the same scan to maximum side $192$ pixels and rerunning
segmentation and landmark extraction. Thus the perturbations arise from the
image-processing pipeline rather than sampled tangent vectors. All $225$
queries lie inside the kernel support throughout the sweep, and recall uses
the same terminal-distance criterion $d\leq0.05h$ as the synthetic
experiment. No parameter is selected from query outcomes.

Figure~\ref{fig:leaf-transition} shows that $E_\geo$ returns every query to
its center throughout, whereas $E_\vc$ returns no query below $\beta=10$
and every query above it. At equality, the implemented dynamics converge to
the center through higher-order terms, but its Hessian is singular; the
open marker therefore does not count as a nondegenerate memory. This is a
real-data replication of the pointwise Ricci transition, not a claim of
superior multi-pattern retrieval.

\begin{figure}[!ht]
\centering
\includegraphics[width=0.72\linewidth]{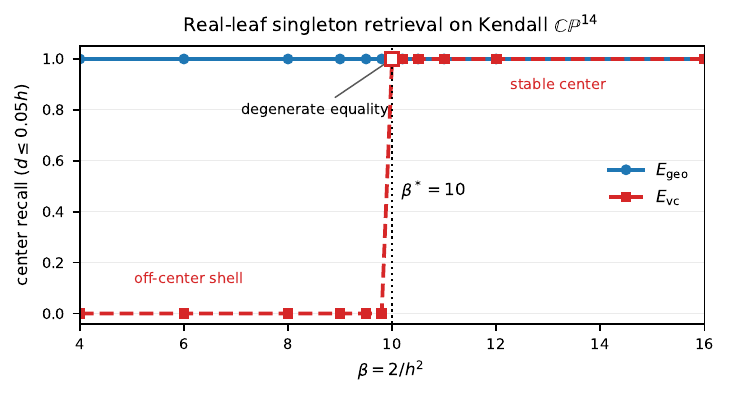}
\caption{Real-leaf singleton retrieval on Kendall
$\mathbb{CP}^{14}$. Each point aggregates $225$ paired
full-resolution/downsampled contour extractions. The exact prediction is
$\beta^*=10$: $E_\geo$ recalls throughout, while $E_\vc$ moves from an
off-center shell below threshold to the center above threshold. The open
marker at equality denotes dynamical convergence with a degenerate Hessian,
which is excluded by the memory definition.}
\label{fig:leaf-transition}
\end{figure}

\subsection{Held-out clustered denoising}
\label{app:denoising}

We simulate clustered observations on $S^{63}$, $SO(3)$, and
$\mathrm{SPD}(3)$ with its affine-invariant metric. There are respectively
$24$, $8$, and $12$ concept centers, none of which is stored. Each bank
contains $k=8$ noisy tangent-space perturbations per center. In each of five
independent repetitions, the centers are fixed across noise levels, while
the bank, four validation queries per center, and eight test queries per
center are sampled independently. Radius $h$ is chosen by validation error
for the two intrinsic energies. We likewise tune projected Euclidean
Epanechnikov retrieval \citep{hoover} and the inverse temperature of the
dot-product Hopfield update \citep{ramsauer}. Nearest stored trial is the
nonparametric baseline, while the true-cluster Karcher mean uses unavailable
membership labels and is only a diagnostic reference.

Figure~\ref{fig:denoise} reports held-out intrinsic error. $E_\geo$ reduces
nearest-neighbor error by $2.2$--$2.7\times$ at every tested setting and
usually reaches the true-cluster Karcher reference. On the sphere and
$SO(3)$, intrinsic and projected ambient methods are nearly
indistinguishable. The corrected spherical model separates at the largest
noise, where validation selects $\beta\approx10$ below the stability
threshold $62/3$: its error is $0.146$, versus $0.117$ for $E_\geo$.
Geometry matters strongly for scale-varied SPD banks. As $\sigma$ increases
from $0.15$ to $0.45$ and $0.8$, projected-Euclidean error grows from
$0.060$ to $0.504$ and $2.17$, while $E_\geo$ obtains $0.054$, $0.155$,
and $0.255$.

\begin{figure}[!ht]
\centering
\includegraphics[width=\linewidth]{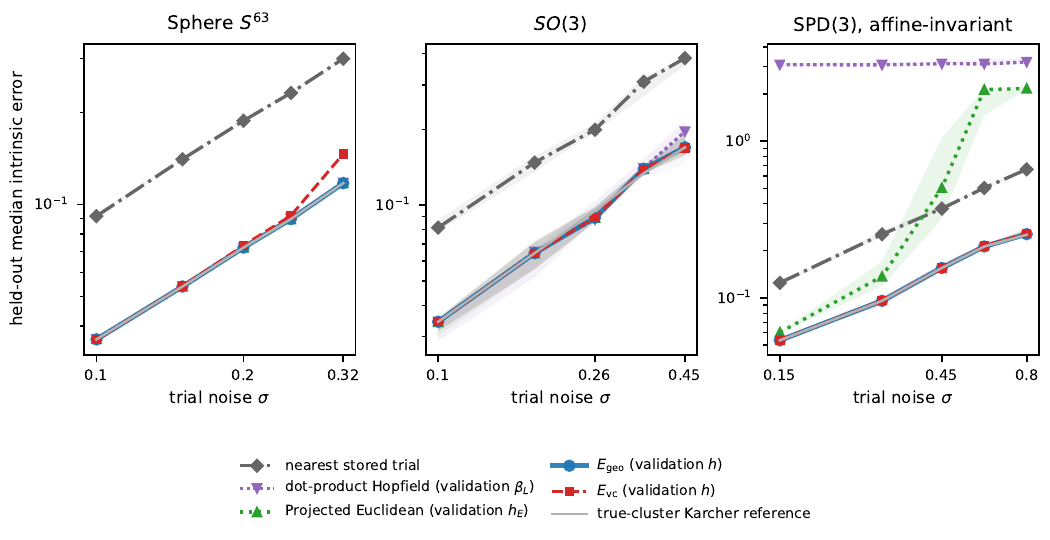}
\caption{Cluster denoising from $k=8$ stored noisy trials per concept.
Curves show median held-out center error over five independently sampled
banks; bands show the repetition interquartile range. Hyperparameters use
disjoint validation queries. The Karcher curve uses true cluster membership
and is not an attainable unlabeled baseline.}
\label{fig:denoise}
\end{figure}

\subsection{Geometry-sensitive asymmetric-noise simulations}
\label{app:geometry-sensitive}

The isotropic sphere and rotation experiments above are a useful
no-penalty control, but they do not separate intrinsic and projected ambient
aggregation: chordal distance is a monotone function of geodesic distance on
both manifolds, and symmetry places the intrinsic and extrinsic population
centers at the same point. Here we remove only that symmetry while keeping
the active-set comparison controlled.

\paragraph{Construction.}
We use $S^2$ and $SO(3)$. The low-dimensional sphere keeps all candidate
corrected-model bandwidths below the isolated-instability boundary,
separating aggregation bias from the high-dimensional transition in
Figure~\ref{fig:sim-detail}. For each independently sampled center $c$ and
random unit tangent direction $u$, one local bank contains $16$ trials near
$\exp_c(a u)$ and $10$ near $\exp_c(-1.6a u)$, each with isotropic tangent
jitter of scale $0.03$. Before jitter, the weighted log-offset is exactly
zero,
\[
 16(a u)+10(-1.6a u)=0,
\]
so $c$ is the intrinsic Fr\'echet center, while nonlinear embedding makes
the projected ambient mean drift as $a$ grows. Each concept is evaluated in
its own bank, isolating aggregation geometry from pattern interference. For
each of five independent repetitions we use $12$ centers, four validation
queries and eight held-out test queries per center; query noise has tangent
scale $0.08$. Radius and Hopfield temperature are selected by validation-set
center error exactly as in the clustered denoising study.

\paragraph{Results.}
Figure~\ref{fig:geometry-sensitive} shows that all averaging methods agree in
the local regime, but their targets separate with increasing spread. At
$a=0.9$ on $S^2$, median held-out error is $0.010$ for $E_\geo$, $0.139$
for $E_\vc$, and $0.229$ for both projected Euclidean Epanechnikov and the
validation-selected dot-product Hopfield update. On $SO(3)$ the corresponding
errors are $0.0067$, $0.087$, and $0.230$. Thus $E_\geo$ reduces the ambient
error by factors of $22$ and $34$, while $E_\vc$ reduces it by factors of
$1.65$ and $2.63$. The sample Karcher reference tracks $E_\geo$, confirming
that the gap is intrinsic-versus-extrinsic center bias rather than failed
optimization.

\begin{figure}[!ht]
\centering
\includegraphics[width=\linewidth]{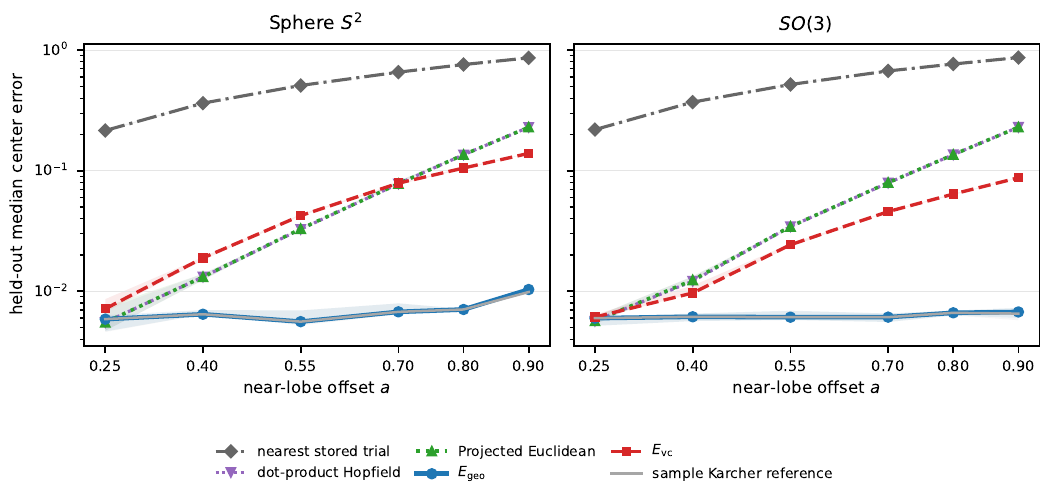}
\caption{Fr\'echet-center recovery under asymmetric manifold noise. Each
bank has $16$ trials near log-offset $+a u$ and $10$ near $-1.6a u$, so the
unjittered intrinsic mean is the center although the projected ambient mean
is biased. Curves show median held-out center error over five independent
repetitions; bands show the repetition interquartile range.}
\label{fig:geometry-sensitive}
\end{figure}

\section{Designed emergent-mode generation}
\label{app:designed-generation}

Unlike denoising, generation has no prescribed target center. A query is a
seed, and its converged output is counted as generated only if it is a stable
mode at least $\tau h$ from every stored pattern. We first instantiate the
maximal construction of Theorem~\ref{thm:emergence}, for which the complete
mode set is known.

\paragraph{Construction and verification.}
At a base point $p$, we place a centered regular simplex in $T_p\M$, scaled
so its vertices have pairwise distance $\sqrt{2}$, and store their images
under $v\mapsto\exp_p(hv)$. We use three vertices on $S^2$ and four on
$SO(3)$. Thus the theorem predicts respectively $2^3-1=7$ and
$2^4-1=15$ modes, of which $4$ and $11$ are non-singleton emergent modes.
For each $h\in\{0.15,0.25,0.35,0.45,0.55\}$ and each intrinsic energy, we
enumerate every nonempty subset, initialize at its sample Karcher mean, and
run retrieval using the full bank. A candidate is accepted only when its
active set equals the proposed subset, its dimensionless stationarity
residual is below $10^{-6}$, and a tangent finite-difference score Hessian
is negative definite. Novelty uses $\tau=1/2$.

We separately test reachability without supplying subset labels. At
$h=0.45$, $5000$ common queries are drawn by choosing an original uniformly,
choosing a radial offset uniformly in $[0.55h,0.99h]$, and choosing a random
tangent direction. This support-annulus proposal is designed to probe overlap
basins; its hit percentages are therefore conditional on the stated proposal,
not uniform volume fractions of the entire manifold.

\paragraph{Results.}
Both $E_\geo$ and $E_\vc$ realize all $7$ spherical modes and all $15$
rotation modes at every tested bandwidth, with all originals retained.
Across the sweep, the largest stationarity residual is
$1.7\times10^{-12}$; the generated points lie within $0.022h$ on $S^2$ and
$0.0064h$ on $SO(3)$ of their tangent-centroid predictions. Blind queries
reach an emergent mode in $43.2\%$ of spherical trials and $33.3\%$ of
rotation trials, discovering all $4$ and all $11$ emergent modes for each
energy. The symmetric construction makes the geodesic and corrected
locations nearly coincide; this experiment tests the existence and
reachability claims rather than a curvature advantage.

\begin{figure}[!ht]
\centering
\includegraphics[width=\linewidth]{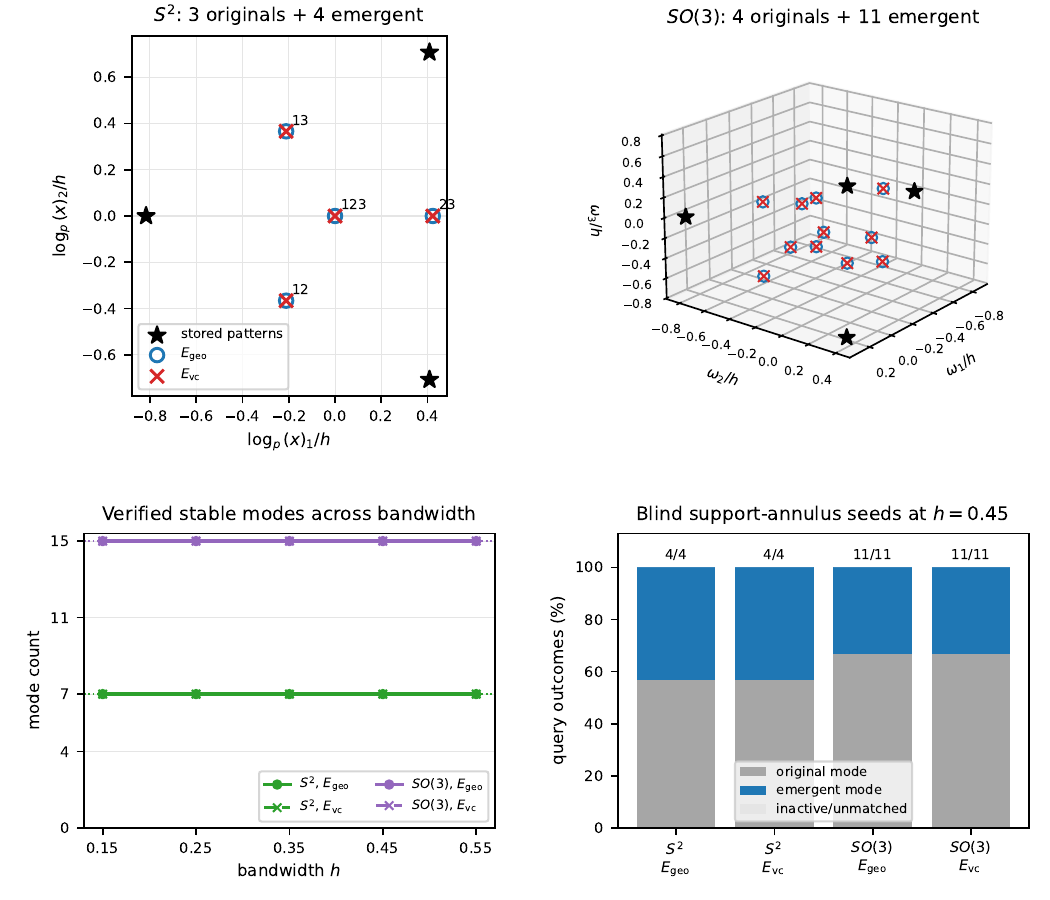}
\caption{Designed generation on $S^2$ and $SO(3)$. Top: stored simplex
vertices (stars) and every verified non-singleton mode in normalized
base-point log coordinates at $h=0.45$. Bottom left: the exact mode counts
persist across the bandwidth sweep. Bottom right: outcomes of $5000$ blind
support-annulus queries; labels give emergent modes reached over emergent
modes present. All original patterns remain stable.}
\label{fig:designed-emergence}
\end{figure}

\section{Quality of generation from retained random banks}
\label{app:random-generation}

The preceding construction deliberately realizes every subset. We next ask
whether additional modes from random data remain plausible under the
distribution that produced the bank. This separates geometric novelty from
statistical quality, following the distinction emphasized for Euclidean
Epanechnikov memories by \citet{hoover}.

\paragraph{Protocol.}
On each manifold we define an equal four-component tangent-normal mixture.
The $S^2$ component centers are tetrahedral, with tangent standard deviation
$0.28$; the $SO(3)$ centers are the identity and rotations of angle $2$
about the three coordinate axes, with tangent standard deviation $0.32$.
Each bank has five independently perturbed points per component. We use fixed
bandwidths $h=0.08$ and $h=0.17$, respectively, and rejection-sample $12$
independent banks conditional only on every inter-pattern distance exceeding
$h$. Hence every original is isolated. Pairs at distance below $2h$ may,
but need not, produce overlap modes.

For each retained bank, candidates are collected from every overlapping-pair
midpoint and from $4000$ blind support-annulus queries generated as above.
Outputs are deduplicated at distance $0.02h$ and must be fixed points with at
least two active patterns and novelty at least $0.5h$; intrinsic candidates
additionally pass a negative-definite score-Hessian check. We report the
principal tangent-normal mixture log density as a percentile among $3000$
independent held-out samples; the neglected probability outside the
injectivity radius is numerically negligible at these scales. A $50$th
percentile output therefore has typical held-out density. Projected Euclidean
Epanechnikov retrieval is included as a local-geometry control.

\paragraph{Results.}
Emergent modes occur in $10/12$ retained $S^2$ banks and all $12/12$
retained $SO(3)$ banks. For $E_\geo$, the median numbers per bank are $3$
and $2$, while median normalized novelty is $0.75$ and $0.83$. Among banks
with an emergent mode, median per-bank quality is the $66$th held-out
percentile on $S^2$ and the $70$th on $SO(3)$, compared with stored-bank
medians at the $58$th and $46$th percentiles and uniform-manifold references
at the $3$rd and $0$th.
Blind queries reach a novel mode in median $2.82\%$ and $1.54\%$ of trials;
the median discovered-mode coverage is $100\%$. Perturbations kept inside
the measured isolation margin recall every original.

At these retention-scale bandwidths, $E_\geo$, $E_\vc$, and projected
Euclidean retrieval agree to plotting precision. The experiment therefore
supports a compact-support emergence claim---new stable and statistically
plausible points coexist with exact storage---but does not claim that
Riemannian geometry alone creates semantic creativity. A decoder-valued
latent experiment would be required for that stronger interpretation.

\begin{figure}[!ht]
\centering
\includegraphics[width=\linewidth]{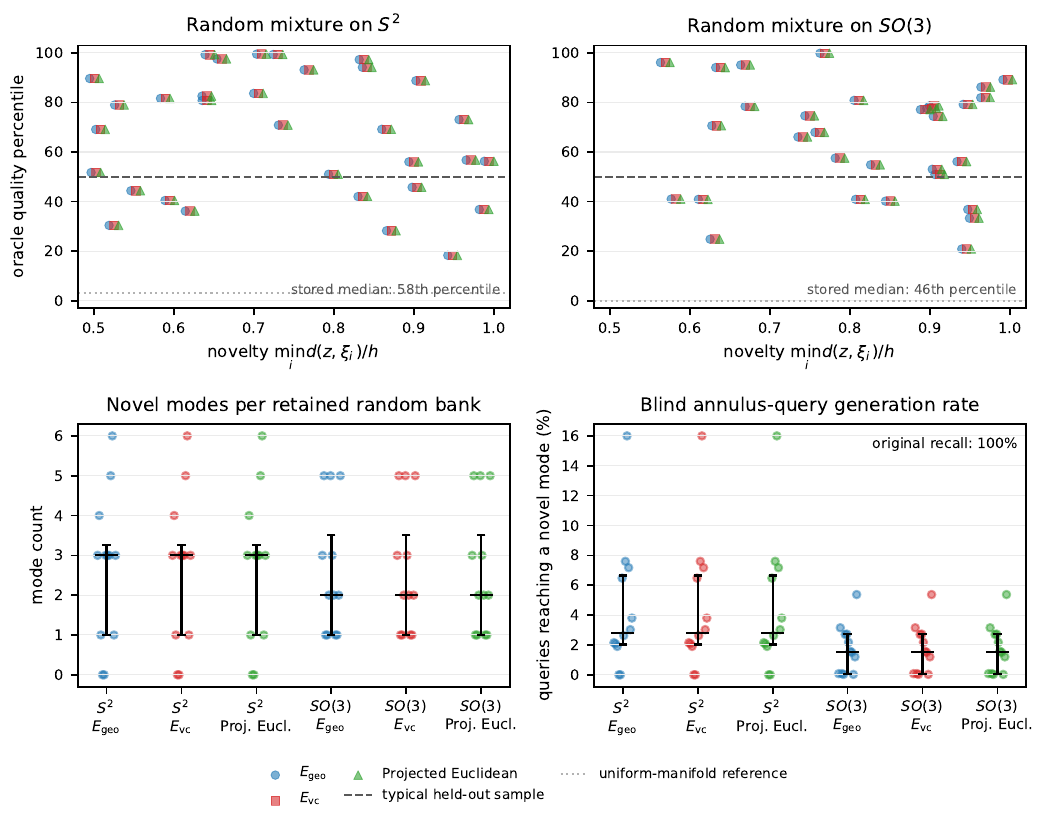}
\caption{Generation from random banks conditioned on retaining every
original. Top: oracle density percentile versus normalized novelty for every
stable generated mode; the dashed line is the density of a typical held-out
sample and the dotted line is the median uniform-manifold reference. Bottom:
mode counts and the fraction of blind support-annulus queries reaching a
novel mode across $12$ retained banks. Zero counts are retained banks with no
overlap-generated mode. All tested original-recall probes succeed.}
\label{fig:generative-quality}
\end{figure}
 \clearpage
\section{Full-WordNet hierarchy retrieval}
\label{app:wordnet}

This appendix gives the complete protocol for the semantic-hierarchy
experiment summarized in Section~\ref{sec:experiments}. The principal
comparison is intrinsic memory against ordinary Euclidean memory on the
same Poincare coordinates, so representation quality is held fixed.

\subsection{Dataset, embeddings, and held-out concepts}

We parse the WordNet 3.0 noun database directly \citep{miller1995wordnet}
and retain all descendants of \texttt{entity.n.01}. The resulting directed
acyclic graph contains $82{,}115$ noun synsets and $743{,}241$ ordered
synset--ancestor pairs in its hypernym closure. For each seed, $128$ leaf
concepts of depth at least seven are held out from embedding training. The
remaining $81{,}987$ concepts are anchors.

We train a $10$-dimensional Lorentz embedding with the sampled-softmax
hierarchy objective used for Poincare representations
\citep{nickel2017poincare}. Each of the five independent runs uses $200$
epochs, $300{,}000$ sampled positive relations per epoch, $50$ filtered
negative relations per positive, batch size $1024$, learning rate $0.5$,
and a $20$-epoch burn-in at one hundredth of that rate. A fixed set of
$1000$ held-out relations measures filtered ranking quality. Across seeds
83--87, mean reciprocal rank is $0.098$, with range $[0.075,0.109]$; mean
rank is $477$.

For each held-out leaf, its clean point is inferred from every available
ancestor relation while anchor points remain fixed. Partial-query points
use nested random subsets containing $25\%$, $50\%$, or $75\%$ of those
relations, with three independent relation orderings per target. Clean and
partial points are optimized for $150$ epochs with batch size $512$, the
same negative count, and the same Lorentz learning rate. The Poincare chart
is obtained exactly from the inferred Lorentz points.

\subsection{Retrieval protocol and statistical unit}

A memory bank contains clean points for a subset of the $128$ held-out
leaves. For each method and bank, let $s_i$ be the nearest-neighbor distance
of stored point $i$. The support radius is the largest strict open radius
for which at least $80\%$ of stored points satisfy $s_i>h$. This depends
only on the bank, never on a query or its target identity. A retrieval is
successful when the target is among the retained originals and the terminal
point lies within $0.05h$ of it.

We compare $E_\geo$, $E_\vc$, and ordinary Euclidean Epanechnikov memory on
the same Poincare coordinates. The latter is called \emph{ambient
Euclidean}; it isolates the memory geometry from the learned
representation. To average over bank composition, the analysis uses $256$
deterministic random banks within each embedding seed. The independent
statistical units remain the five embedding and target-split seeds, and
reported intervals bootstrap their five mean paired differences.

\begin{table}[!ht]
\centering
\small
\setlength{\tabcolsep}{4.5pt}
\begin{tabular}{ccrrrr}
\toprule
Bank & Evidence & Ambient & $E_\geo$ & $E_\vc$ &
  $E_\vc-\mathrm{Ambient}$ \\
\midrule
16 & $25\%$ & .339 & .316 & .419 & $.080\ [.055,.101]$ \\
16 & $50\%$ & .560 & .524 & .617 & $.057\ [.036,.078]$ \\
16 & $75\%$ & .717 & .686 & .757 & $.041\ [.028,.054]$ \\
\bottomrule
\end{tabular}
\caption{Exact target recovery on the full WordNet noun hierarchy. Entries
are five-seed means after averaging $256$ banks within seed. Brackets give
the $95\%$ seed-bootstrap interval for the paired difference.}
\label{tab:wordnet-retrieval}
\end{table}

At bank size $16$, all five seed-level $E_\vc$ differences are positive.
The mean $E_\vc-E_\geo$ improvements are $0.103$, $0.093$, and $0.072$,
with seed-bootstrap intervals $[0.083,0.118]$, $[0.075,0.108]$, and
$[0.059,0.083]$. These paired improvements show that the volume correction
contributes beyond the use of hyperbolic distance alone.

 \section{EEG covariance-state storage and retrieval}
\label{app:eeg}

We test whether the memory definitions lead to measurable behavior on real
manifold-valued observations. The represented object is an EEG spatial
covariance descriptor, not a time-domain signal. This auxiliary real-data
study is reported only in the appendix; the main experimental section
focuses on the full WordNet hierarchy.

\paragraph{Data and covariance representation.}
We use BCI Competition IV-2a \citep{bci2a}, loaded through MOABB
\citep{moabb}. It contains nine subjects, two recording sessions, $22$ EEG
channels, four motor-imagery classes, and $288$ trials per subject and
session. Following the covariance-based Riemannian BCI representation
\citep{barachant2012}, each $2$--$6$ second post-cue epoch is band-pass
filtered to $8$--$30$ Hz. For a centered channel-by-time matrix $X$, we use
\[
 C=(1-\rho)\frac{XX^\top}{T-1}
   +\rho\,\frac{\operatorname{tr}(XX^\top/(T-1))}{22}I,
 \qquad \rho=0.05,
\]
then divide by $\det(C)^{1/22}$. The resulting covariance shapes lie on the
unit-determinant part of affine-invariant $\operatorname{SPD}(22)$. The full
four-second covariance is a stored target. Four evenly positioned
covariances at each of one, two, and three seconds provide real
partial-observation queries from the same physical trial.

\subsection{Storage and paired retrieval}

\paragraph{Protocol.}
Within held-out session 2, deterministic class-balanced farthest-first
selection gives nested banks of $32$, $64$, and $96$ full-trial
covariances. The same first eight targets per class are queried at every bank
size, so adding distractors does not change target difficulty. The displayed
results use $96$ stored covariances and all $32\times4$ paired queries.

For a requested retention fraction $r$, the radius is the largest value
determined from bank-to-bank distances alone for which at least fraction
$r$ of originals have no neighbor inside their open support. No query,
target identity, or class label tunes this radius. We report the
\emph{target singleton-basin fraction}: the fraction of queries for which
the true target is the unique active memory. For $E_\geo$, every such query
recalls its target exactly in one step by
Theorem~\ref{prop:descent}. $E_\vc$ has the same support event, and an
isolated target is stable on affine-invariant SPD because its Ricci curvature
is nonpositive; the statement is not a one-step guarantee for the corrected
dynamics.

The compact controls use matrix-log Euclidean distance or Frobenius distance
in the native sensor coordinates. Affine-invariant nearest neighbor is
included as a radius-free readout reference, not as a mode-generating energy.

\paragraph{Results.}
At $95\%$ requested storage retention, the mean target singleton-basin
fractions for intrinsic support are $0.018$, $0.683$, and $0.940$ for one-,
two-, and three-second queries. At two seconds, native log-Euclidean support
obtains $0.658$ and native projected Euclidean support obtains $0.080$. The
corresponding affine-nearest-neighbor target recall is $0.941$, but it
supplies neither a compact-support retention event nor additional modes.

Thus the intrinsic model improves the native log-Euclidean singleton-basin
fraction by $0.025$ on average and the native projected Euclidean result by
$0.603$. For the log-Euclidean comparison, the subject-paired bootstrap
$95\%$ interval is $[0.008,0.043]$ and the Holm-adjusted exact sign-flip
$p$-value is $0.0469$. The increase from one- to three-second queries
quantifies how longer partial observations enlarge target singleton basins
while retaining $95\%$ of the $96$ trial memories.

We separately apply common determinant-one channel transforms with condition
numbers from $1$ to $100$. Intrinsic recall is unchanged at $0.683$ by
construction. Native log-Euclidean recall decreases from $0.658$ to $0.583$,
and native projected Euclidean recall falls from $0.080$ to $0.055$. This
demonstrates an inherent equivariance advantage over native-coordinate flat
metrics.

\raggedbottom
\begin{figure}[H]
\centering
\includegraphics[width=\linewidth]{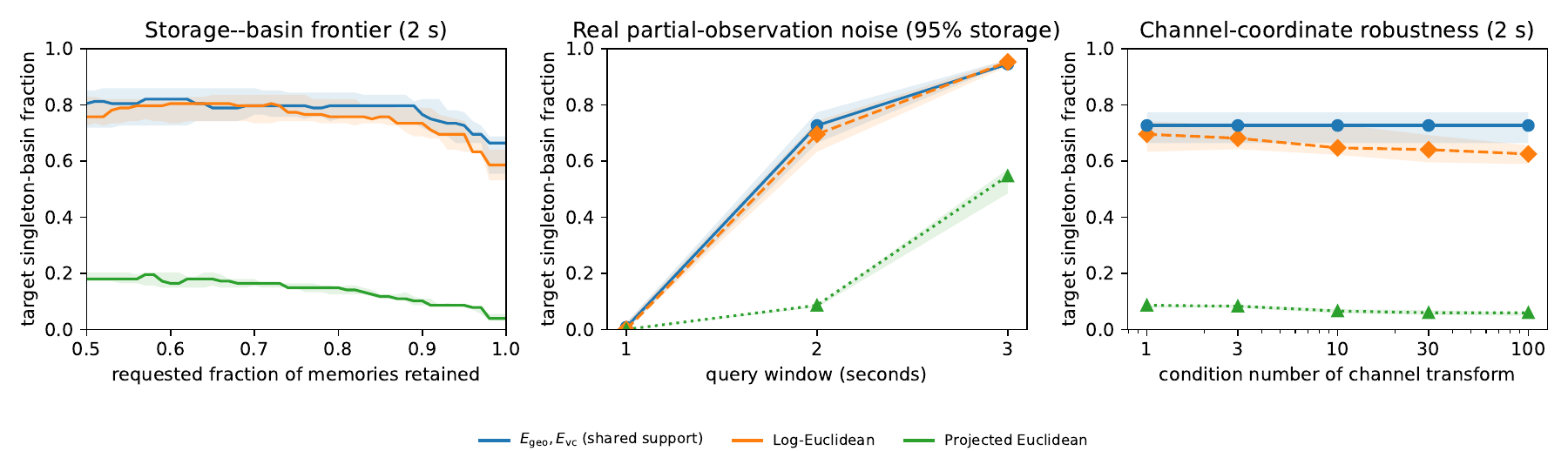}
\caption{Paired EEG covariance storage and retrieval on held-out session 2.
Left: requested original retention versus the fraction of real two-second
queries whose target is the unique active memory. Middle: query duration at
$95\%$ requested retention. Right: a common invertible channel
reparameterization. Curves and bands are the subject median and IQR. The
singleton event certifies one-step exact recall for $E_\geo$ and identifies
an isolated corrected basin for $E_\vc$; it is not claimed as one-step
corrected recall. Native log-Euclidean is the strongest flat control, while
native projected Euclidean support is substantially weaker.}
\label{fig:eeg-memory}
\end{figure}

\paragraph{Scope.}
This experiment tests content-addressable covariance storage and paired
retrieval, not end-to-end BCI decoding. The target singleton event directly
matches the paper's retrieval guarantee, while the channel-transform study
tests the expected intrinsic equivariance.

\section{Local geometric identities}\label{app:geometry}

We use the curvature convention for which the unit sphere has $\Ric=(m-1)g$. The basic geometric facts about normal coordinates and intrinsic means are compatible with the conventions in~\cite{pelletier,afsari}. The derivatives needed here are derived below. Lemma~\ref{lem:jets} is local and remains valid on a complete noncompact manifold within normal domains; compactness enters only the subsequent uniform bounds.

\begin{lemma}[Distance derivatives and volume-density jet]\label{lem:jets}
For $q_p(x)=d(x,p)^2/2$ in the normal domain,
\[
 \grad q_p(x)=-\log_xp,\qquad dq_p(p)=0,\qquad\Hess q_p(p)=g_p.
\]
Moreover,
\begin{equation}
 \theta_p(\exp_p v)=1-\tfrac16\Ric_p(v,v)+O(\norm v^3),\label{eq:thetaseries}
\end{equation}
so, for $a_p=\theta_p^{-1}$,
\[
 a_p(p)=1,\quad da_p(p)=0,\quad\Hess a_p(p)=\tfrac13\Ric_p.
\]
For points in the common short normal domain, $\theta_p(x)=\theta_x(p)$.
\end{lemma}
\begin{proof}
Let $x(s)$ be a variation of $x$ with derivative $w$, and let $\gamma_s:[0,1]\to\M$ be the minimizing geodesic from $x(s)$ to $p$, parameterized with constant speed. In the normal domain this is a smooth variation. Its energy is $q_p(x(s))$. If $V=\partial_s\gamma_s|_{s=0}$, differentiation and integration by parts give
\[
 dq_p(x)[w]=\int_0^1\ip{D_tV}{\dot\gamma_0}\,dt
 =\left[\ip V{\dot\gamma_0}\right]_0^1
 -\int_0^1\ip V{D_t\dot\gamma_0}\,dt
 =-\ip w{\log_xp}.
\]
Here $V(1)=0$ and $D_t\dot\gamma_0=0$. This proves the gradient formula. In normal coordinates at $p$, $q_p(\exp_pv)=\norm v^2/2$, giving its value, differential, and Hessian at $p$.

For the volume jet, fix a unit vector $u$ and let $\gamma(t)=\exp_p(tu)$. Express Jacobi fields in a parallel orthonormal frame. The Jacobi matrix satisfies
\[
 A''(t)+R(t)A(t)=0,\qquad A(0)=0,\quad A'(0)=I,
\]
where $R(t)z=R(z,\dot\gamma(t))\dot\gamma(t)$. Thus $A''(0)=0$ and $A'''(0)=-R(0)$, whence
\[
 A(t)=tI-\tfrac16t^3R(0)+O(t^4).
\]
The differential of the exponential map at $tu$, represented in this frame, is $A(t)/t$. Expanding the determinant by its multilinearity gives
\[
 \det(A(t)/t)=1-\tfrac16t^2\operatorname{tr}R(0)+O(t^3)
 =1-\tfrac16t^2\Ric_p(u,u)+O(t^3).
\]
Varying $u$ proves~\eqref{eq:thetaseries}. Taking the reciprocal gives
$a_p(\exp_pv)=1+\Ric_p(v,v)/6+O(\norm v^3)$ and hence the asserted derivatives.

For symmetry, let $r=d(p,x)$ and use a parallel frame along the short geodesic from $p$ to $x$. In addition to $A$, define the reversed Jacobi matrix $B$ by
\[
 B''+RB=0,\qquad B(r)=0,\qquad B'(r)=-I.
\]
Since $R$ is symmetric,
\[
 \frac d{dt}(A^TB'-A'^TB)=A^TB''-A''{}^TB=0.
\]
At $t=0$ this constant matrix is $-B(0)$; at $t=r$ it is $-A(r)^T$. Therefore $B(0)=A(r)^T$. The two volume densities are respectively $\det A(r)/r^m$ and $\det B(0)/r^m$, with positive determinants before conjugacy. Their equality proves the result.
\end{proof}

\begin{assumption}[Convex supports and corrected-kernel concavity]\label{ass:convex}
Each $B_h(\xi_i)$ is strongly geodesically convex, and
\[
 \Hess q_i\succeq c g\quad\text{on }B_h(\xi_i),\qquad c>0.
\]
For the corrected energy also assume, on these balls,
\[
 a_i\ge a_*>0,\qquad\norm{\grad a_i}\le L,\qquad\norm{\Hess a_i}\le H,
 \qquad \mu:=\beta a_*c-2\beta hL-H>0.
\]
The inequalities for tensors are pointwise; their norms are operator norms induced by $g$.
\end{assumption}

The derivative bound is a sufficient condition for the corrected component concavity used in Theorem~\ref{thm:active}. The isolated-pattern criterion of Theorem~\ref{thm:isolated} does not require it.

\begin{lemma}[Uniform small-bandwidth bounds]\label{lem:small}
On a fixed smooth compact manifold there is $h_0>0$ such that for $0<h<h_0$ all radius-$h$ balls are strongly convex, and uniformly for $d(x,p)<h$,
\[
 \Hess q_p=g+O(h^2),\quad a_p=1+O(h^2),\quad
 \norm{\grad a_p}=O(h),\quad\norm{\Hess a_p}=O(1).
\]
Consequently Assumption~\ref{ass:convex} holds, and its $\mu$ can be bounded below by $c_0h^{-2}$ for some $c_0>0$.
\end{lemma}
\begin{proof}
The short-geodesic squared distance and the volume density are smooth functions of both endpoints in a neighborhood of the diagonal. Normal-coordinate Christoffel symbols vanish at their centers, and their coefficients and derivatives have uniform bounds on a fixed small neighborhood of the diagonal by compactness. Since $q_p=\norm v^2/2$ in coordinates centered at $p$,
\[
 (\Hess q_p)_{jk}=\delta_{jk}-\Gamma_{jk}^{\ell}(v)v_\ell
 =\delta_{jk}+O(\norm v^2),
\]
where $\Gamma(v)=O(\norm v)$. The metric itself is $I+O(\norm v^2)$, so this gives the intrinsic Hessian estimate. The remaining estimates follow from the smooth Taylor expansion in Lemma~\ref{lem:jets}, with uniform remainder bounds. Smoothness also bounds the derivatives of that expansion needed for the gradient and Hessian estimates.

For completeness, uniformly small balls are strongly convex as follows. Fix a uniform normal radius $r_0>0$ on which $\Hess q_p\succ0$ for every center $p$. Choose $h<r_0/3$ and $2h<\inj(\M)$. Two points in $B_h(p)$ have distance less than $2h$, so their minimizing geodesic is unique. Every point on this geodesic has distance less than $3h$ from $p$ by the triangle inequality. Along the geodesic, $q_p$ has positive second derivative unless the geodesic is constant. Its value is bounded above by the larger endpoint value by convexity of this one-dimensional function. Therefore the geodesic stays in $B_h(p)$.

In Assumption~\ref{ass:convex}, one may take $c\ge1/2$, $a_*\ge1/2$, $L\le C_1h$, and $H\le C_2$ once $h$ is small. Since $\beta=2/h^2$,
\[
 \mu\ge\frac1{2h^2}-4C_1-C_2.
\]
Shrinking $h_0$ further gives $\mu\ge1/(4h^2)$.
\end{proof}

\section{Supporting density-estimation results}\label{app:kde}
The volume-corrected estimator and its bias are classical manifold KDE constructions \citep{pelletier,henry}; geometric bias is also studied by \citet{kimpark}. We specialize the expansions to the Epanechnikov kernel and distinguish raw smoothing from normalization. These expansions are supporting results, not a priority claim for manifold KDE.

Let $\omega_m$ be the Euclidean unit-ball volume and set
\[
 K(u)=c_m(1-\norm u^2)_+,\qquad c_m=\frac{m+2}{2\omega_m}.
\]
For independent patterns with density $f$ relative to $dV$, define
\[
 \widehat f_\geo(x)=\frac1{Nh^m}\sum_iK(\log_{\xi_i}x/h),\qquad
 \widehat f_\vc(x)=\frac1{Nh^m}\sum_i a_i(x)K(\log_{\xi_i}x/h).
\]
A summand is zero when $d(x,\xi_i)\ge h$. Write
\[
 \widehat Z_h=\int_\M\widehat f_\geo\,dV,\qquad
 \widetilde f_\geo=\widehat f_\geo/\widehat Z_h,\qquad
 \widetilde f_\vc=\widehat f_\vc.
\]
Both $\widetilde f_T$ are probability densities. The distinction between $\widehat f_\geo$ and $\widetilde f_\geo$ is important: the former is a raw smoother.

\begin{theorem}[Bias and risk of normalized density estimators]\label{thm:kde}
Suppose $f\in C^4(\M)$ and $h\downarrow0$. Uniformly in $x$,
\begin{align}
 \E\widehat f_\vc(x)&=f(x)+\frac{h^2}{2(m+4)}\Delta f(x)+O(h^4),\label{eq:biasvc}\\
 \E\widehat f_\geo(x)&=f(x)+\frac{h^2}{2(m+4)}
 \left[\Delta f(x)-\tfrac13\Scal(x)f(x)\right]+O(h^4).\label{eq:biasgeo}
\end{align}
Here $\Delta f=\operatorname{tr}_g\Hess f$ and $\Scal=\operatorname{tr}_g\Ric$.
For the raw and corrected smoothers,
\[
 \int_\M\Var(\widehat f_T(x))\,dV(x)
 =\frac{R_m}{Nh^m}(1+o(1)),\qquad
 R_m=\frac{2(m+2)}{(m+4)\omega_m}.
\]
Let $\bar s_f=\int\Scal f\,dV$, $D_\vc f=\Delta f$, and
$D_\geo f=\Delta f-(\Scal-\bar s_f)f/3$. For the \emph{normalized densities},
\begin{equation}
 \E\norm{\widetilde f_T-f}_{L^2}^2
 =\frac{h^4}{4(m+4)^2}\norm{D_Tf}_{L^2}^2
 +\frac{R_m}{Nh^m}+o\!\left(h^4+\frac1{Nh^m}\right).
 \label{eq:mise}
\end{equation}
In particular the upper rate is $O(N^{-4/(m+4)})$ at $h\asymp N^{-1/(m+4)}$.
\end{theorem}

\begin{proof}
We first compute the kernel moments, then the bias and variance, and finally control the random normalizer.

\paragraph{Kernel moments and mass.}
Polar integration gives
\[
 \int_{\R^m}(1-\norm u^2)_+\,du
 =m\omega_m\left(\frac1m-\frac1{m+2}\right)=\frac{2\omega_m}{m+2}.
\]
Thus $\int K=1$. Reflection in a coordinate hyperplane gives $\int u_jK=0$ and $\int u_ju_kK=0$ for $j\ne k$. By permutation symmetry, all diagonal second moments are equal, and
\begin{align*}
 \int u_j^2K(u)\,du
 &=\frac1m\int\norm u^2K(u)\,du\\
 &=c_m\omega_m\left(\frac1{m+2}-\frac1{m+4}\right)=\frac1{m+4}.
\end{align*}
Also,
\begin{align*}
 \int K(u)^2\,du
 &=c_m^2m\omega_m\left(\frac1m-\frac2{m+2}+\frac1{m+4}\right)\\
 &=c_m^2\frac{8\omega_m}{(m+2)(m+4)}
 =\frac{2(m+2)}{(m+4)\omega_m}=R_m.
\end{align*}
For a fixed center $y$, the substitution $x=\exp_y(hu)$ gives
\[
 \int_\M h^{-m}\theta_y(x)^{-1}K(\log_yx/h)\,dV(x)
 =\int_{\norm u<1}K(u)\,du=1.
\]
Hence every corrected component, and its sample average, is normalized. The raw components are nonnegative and have positive integrals, so division by $\widehat Z_h$ also gives a density.

\paragraph{Bias.}
At a fixed query $x$, use $y=\exp_x(hu)$ and the volume-density symmetry of Lemma~\ref{lem:jets}. For the corrected estimator, the two Jacobian factors cancel:
\[
 \E\widehat f_\vc(x)=\int K(u)f(\exp_x(hu))\,du.
\]
Taylor expansion along the radial geodesic, uniformly for $x\in\M$ and $\norm u\le1$, gives
\[
 f(\exp_x(hu))=f(x)+h\,df_x(u)+\tfrac12h^2\Hess f_x(u,u)
 +h^3P_{3,x}(u)+O(h^4),
\]
where $P_{3,x}$ is a homogeneous cubic polynomial. Its integral and the linear term vanish because they are odd. The computed second moments therefore give~\eqref{eq:biasvc}.

For the raw smoother,
\[
 \E\widehat f_\geo(x)=\int K(u)f(\exp_x(hu))\theta_x(\exp_x(hu))\,du.
\]
The smooth volume-density expansion is
\[
 \theta_x(\exp_x(hu))=1-\tfrac16h^2\Ric_x(u,u)+h^3Q_{3,x}(u)+O(h^4),
\]
where $Q_{3,x}$ is a homogeneous cubic. All first- and third-degree terms in the product are odd. The second-degree term is
$h^2\Hess f_x(u,u)/2-h^2f(x)\Ric_x(u,u)/6$.
Tracing it against $\int u_ju_kK=\delta_{jk}/(m+4)$ yields~\eqref{eq:biasgeo}.

\paragraph{Integrated variance.}
Write a single normalized-scale summand as $Y_T(x,X)$, so that
$\widehat f_T=N^{-1}\sum_iY_T(x,X_i)$. Independence gives
\[
 \int\Var(\widehat f_T(x))\,dV(x)
 =\frac1N\E\int Y_T(x,X)^2\,dV(x)
 -\frac1N\int(\E Y_T(x,X))^2\,dV(x).
\]
Conditional on $X=y$, the geodesic component has squared integral
\[
 h^{-m}\int K(u)^2\theta_y(\exp_y(hu))\,du
 =h^{-m}\bigl(R_m+O(h^2)\bigr).
\]
For the corrected component the remaining factor is $\theta_y^{-1}$ instead of $\theta_y$, yielding the same estimate. The bounds are uniform in $y$. The squared-mean term is $O(1/N)$ by the bias expansions and compactness, and is $o(1/(Nh^m))$. This proves the variance assertion. For $T=\vc$, squaring the uniform bias expansion and integrating gives
\[
 \int(\E\widehat f_\vc-f)^2\,dV
 =\frac{h^4}{4(m+4)^2}\norm{\Delta f}_{L^2}^2+O(h^6),
\]
which proves~\eqref{eq:mise} for the corrected density.

\paragraph{Population normalization.}
Let $m_h=\E\widehat f_\geo$ and $Z_h^{\mathrm{pop}}=\int m_h\,dV$.
The divergence theorem gives $\int\Delta f\,dV=0$ because $\M$ has no boundary. Integrating~\eqref{eq:biasgeo} therefore yields
\[
 Z_h^{\mathrm{pop}}=1-\frac{h^2\bar s_f}{6(m+4)}+O(h^4).
\]
For a scalar $z=O(h^2)$, the identity $(1+z)^{-1}=1-z+z^2/(1+z)$ gives its reciprocal expansion. Multiplying by~\eqref{eq:biasgeo} gives
\begin{equation}
 \frac{m_h(x)}{Z_h^{\mathrm{pop}}}
 =f(x)+\frac{h^2}{2(m+4)}
 \left[\Delta f(x)-\frac{\Scal(x)-\bar s_f}{3}f(x)\right]+O(h^4).
 \label{eq:popnormalizedbias}
\end{equation}
Set $U_h=\widehat f_\geo/Z_h^{\mathrm{pop}}$. This is a random function with deterministic denominator. Its squared bias follows from~\eqref{eq:popnormalizedbias}, and its integrated variance is the raw variance divided by $(Z_h^{\mathrm{pop}})^2=1+O(h^2)$. Hence
\begin{equation}
 \E\norm{U_h-f}_{L^2}^2
 =\frac{h^4}{4(m+4)^2}\norm{D_\geo f}_{L^2}^2
 +\frac{R_m}{Nh^m}+o\!\left(h^4+\frac1{Nh^m}\right).
 \label{eq:poprisk}
\end{equation}

\paragraph{The random sample normalizer.}
It remains to compare $U_h$ with the actual density $\widetilde f_\geo$; we do not interchange expectation and division. The mass of a raw component centered at $y$ is
\[
 z_h(y)=\int Y_\geo(x,y)\,dV(x)
 =1-\frac{h^2\Scal(y)}{6(m+4)}+O(h^4),
\]
by the same kernel moments and the volume expansion. This is uniform in $y$. In particular both $\widehat Z_h$ and $Z_h^{\mathrm{pop}}$ lie in $[1/2,2]$ for small $h$, for every sample.
Put $A_i=z_h(X_i)-\E z_h(X_i)$ and $D=\widehat Z_h-Z_h^{\mathrm{pop}}=N^{-1}\sum_i A_i$. Then $\E A_i=0$ and $|A_i|\le C h^2$. Conditional on $X_i,X_j$, the remaining $A_\ell$ remain independent and centered. The conditional mean of $D$ has magnitude at most $2Ch^2/N$, and its conditional variance is at most $C^2h^4/N$. Enlarging the constant handles $i=j$ and $N=1$, so
\[
 \E[D^2\mid X_i,X_j]\le C' h^4/N.
\]
The factors $Y_\geo(x,X_i)Y_\geo(x,X_j)$ are nonnegative. Multiplying by them, integrating, and summing over $i,j$ gives
\begin{align*}
 \E\bigl[D^2\norm{\widehat f_\geo}_{L^2}^2\bigr]
 &=\frac1{N^2}\sum_{i,j}\E\int D^2Y_\geo(x,X_i)Y_\geo(x,X_j)\,dV(x)\\
 &\le \frac{C'h^4}{N}\,\E\norm{\widehat f_\geo}_{L^2}^2
 \le \frac{C''h^4}{N}\left(1+\frac1{Nh^m}\right).
\end{align*}
The last bound follows by adding the integrated variance and squared mean already established. Since both denominators are at least $1/2$,
\begin{align*}
 e_h:=\E\norm{\widetilde f_\geo-U_h}_{L^2}^2
 &\le16\E\bigl[D^2\norm{\widehat f_\geo}_{L^2}^2\bigr]\\
 &\le\frac{C'''h^4}{N}\left(1+\frac1{Nh^m}\right)
 =o\!\left(\frac1{Nh^m}\right).
\end{align*}
For the last equality, dividing its two terms by $(Nh^m)^{-1}$ gives $O(h^{m+4})$ and $O(h^4/N)$, both tending to zero. Cauchy--Schwarz now gives
\[
 \left|\E\norm{\widetilde f_\geo-f}_{L^2}^2
 -\E\norm{U_h-f}_{L^2}^2\right|
 \le 2\sqrt{\E\norm{U_h-f}_{L^2}^2}\sqrt{e_h}+e_h
 =o\!\left(h^4+\frac1{Nh^m}\right).
\]
Combining with~\eqref{eq:poprisk} proves~\eqref{eq:mise} for the normalized geodesic density. Finally, at $h\asymp N^{-1/(m+4)}$, both $h^4$ and $(Nh^m)^{-1}$ have order $N^{-4/(m+4)}$. This is an upper-rate statement; if the leading bias coefficient vanishes it need not be an optimal bandwidth.
\end{proof}

\paragraph{Relation to the energy background.}
These statistical statements concern the kernel averages, not an arbitrary uniform background added to the energy. With fixed $\epsilon>0$, the normalized score is a mixture with a uniform density. For the corrected energy its two unnormalized masses are $\epsilon\vol(\M)$ and $Nh^m/c_m$. To transfer the displayed density-risk conclusions to that score, set $\epsilon=0$ or make the background negligible at the required risk scale. Merely sending $\beta$ to infinity, or proving mean-square density consistency, does not prove convergence of the number of energy minima.

\section{Proofs of exact retrieval and the LSE comparison}\label{app:retrieval}

\begin{proof}[Proof of Theorem~\ref{thm:isolated} and the single-kernel assertion in Theorem~\ref{prop:descent}]
For any smooth positive $S$, direct differentiation gives
\begin{equation}
 \grad(-\beta^{-1}\log S)=-\frac{\grad S}{\beta S},\qquad
 \Hess(-\beta^{-1}\log S)=-\frac{\Hess S}{\beta S}
 +\frac{dS\otimes dS}{\beta S^2}.\label{eq:loghessian}
\end{equation}
At a critical point the last term vanishes.

The strict separation from $\xi_i$ means that only its own kernel is active in some neighborhood of $\xi_i$. Indeed, if $N>1$, choose the neighborhood radius smaller than both $h$ and $\min_{j\ne i}(d(\xi_i,\xi_j)-h)$; the triangle inequality excludes every other support. If $N=1$, choose any radius smaller than $h$. This argument and the local identities in Lemma~\ref{lem:jets} do not use compactness, proving the noncompact scope of Theorem~\ref{thm:isolated}. In that neighborhood,
\[
 S_\geo=\epsilon+1-\beta q_i,\qquad
 S_\vc=\epsilon+a_i(1-\beta q_i).
\]
By Lemma~\ref{lem:jets}, both have value $1+\epsilon$ and zero differential at $\xi_i$. Their Hessians there are
\[
 \Hess S_\geo=-\beta g,\qquad
 \Hess S_\vc=\tfrac13\Ric-\beta g.
\]
For the second equality, the product-rule cross terms vanish because $da_i=dq_i=0$ at the center. Equation~\eqref{eq:loghessian} now gives~\eqref{eq:ricci}. The corrected Hessian is positive definite exactly when
$g(w,w)-\Ric(w,w)/(3\beta)>0$ for every nonzero $w$, which is equivalent to~\eqref{eq:threshold}.

For the retrieval assertion, let $s=\min_{i\ne j}d(\xi_i,\xi_j)$ and assume~\eqref{eq:retrievalradius}. If $x\in B_\Delta(\xi_i)$, then $d(x,\xi_i)<\Delta<h$. For $j\ne i$,
\[
 d(x,\xi_j)\ge d(\xi_i,\xi_j)-d(x,\xi_i)>s-\Delta\ge h.
\]
Thus the own kernel is active and all others are inactive. Lemma~\ref{lem:jets} and differentiation of the logarithm give
\begin{equation}
 \grad E_\geo(x)=-\frac{\log_x\xi_i}{\epsilon+1-\beta d(x,\xi_i)^2/2}.
 \label{eq:singlegradient}
\end{equation}
Its denominator $t(x)=S_\geo(x)$ is positive and is available from the query score. Therefore
\[
 \exp_x(-t(x)\grad E_\geo(x))=\exp_x(\log_x\xi_i)=\xi_i.
\]
This is the stated exact step, with no approximation of the exponential map.
\end{proof}

\paragraph{Finite temperature and curvature examples.}
For any finite distinct configuration, $s>0$ and compactness bounds the Ricci eigenvalues. Thus the finite choice
\[
 \beta>\max\left\{2/s^2,\ 2/\inj(\M)^2,\ \tfrac13\max_i\lambda_{\max}(\Ric_{\xi_i})\right\}
\]
stores all patterns for both models. In negative Ricci curvature, $-\Ric/(3\beta)$ is positive definite, so the correction increases every Hessian eigenvalue relative to the geodesic model. On the complete hyperbolic space $\mathbb H^m$ of constant sectional curvature $-1$, let $w$ be unit length and extend it to an orthonormal basis $w,e_2,\ldots,e_m$. The definition of Ricci curvature gives
\[
 \Ric(w,w)=\sum_{j=2}^m\operatorname{sec}(\operatorname{span}\{w,e_j\})=-(m-1).
\]
Homogeneity of this quadratic form gives $\Ric=-(m-1)g$. Substitution into~\eqref{eq:ricci} yields
\[
 \Hess E_\vc(\xi_i)=\frac{1+(m-1)/(3\beta)}{1+\epsilon}\,g_{\xi_i}.
\]
Hyperbolic space has no cut locus, so every finite $h$ is admissible for this isolated-pattern statement. This pointwise example does not extend the compact-manifold random limits to noncompact sampling laws.

On the unit sphere, $\theta_p(x)=(\sin r/r)^{m-1}$ for $r=d(p,x)<\pi$, and a corrected single-kernel score has radial part
\[
 k(r)=(1-\beta r^2/2)(r/\sin r)^{m-1}
 =1+\tfrac12\bigl((m-1)/3-\beta\bigr)r^2+O(r^4).
\]
The expansion follows from $\sin r/r=1-r^2/6+O(r^4)$, raising its reciprocal to power $m-1$, and multiplying by $1-\beta r^2/2$. Below the stability threshold, its derivative is positive for sufficiently small positive $r$. The score is zero at $h<\pi$ and positive inside, so continuity gives a maximum at some radius strictly between zero and $h$. Rotational symmetry makes that maximum an off-center shell. Its tangential Hessian directions vanish, so this shell is not a set of nondegenerate memories under our definition.

For comparison, define the geodesic log-sum-exponential energy
\[
 E_{\mathrm{LSE}}(x)=-\beta^{-1}\log\sum_{i=1}^N e^{-\beta d(x,\xi_i)^2/2}.
\]
The next statement is the intrinsic analogue of the distinction between approximate LSE retrieval and exact compact-support storage discussed by \citet{ramsauer,hoover}. It is an almost-sure result for a fixed parameter, not a claim about every symmetric configuration or every parameter selected from the data.
\begin{proposition}[Generic obstruction for geodesic LSE]\label{prop:lse}
Fix $N\ge2$ and a deterministic finite $\beta>0$. For independent patterns drawn from any distribution absolutely continuous with respect to $dV$, almost surely no original pattern is stationary for $E_{\mathrm{LSE}}$. In particular, global emergence in the memory definition in Section~\ref{sec:model} fails almost surely at that fixed $\beta$.
\end{proposition}

\begin{proof}[Proof of Proposition~\ref{prop:lse}]
For a fixed center $x$, its cut locus has Riemannian volume zero. Thus, with probability one, none of the other finitely many independent patterns lies in the cut locus of $X_i$, and the energy is smooth at each original. Differentiation there gives
\[
 \grad E_{\mathrm{LSE}}(x)=
 -\frac{\sum_j e^{-\beta d(x,X_j)^2/2}\log_xX_j}
 {\sum_j e^{-\beta d(x,X_j)^2/2}}.
\]
Fix $i\ne j$. Condition on $X_i=x$ and all patterns except $X_j$. The contribution from the own pattern is zero, so stationarity requires
\[
 e^{-\beta d(x,X_j)^2/2}\log_xX_j=C
\]
for a fixed vector $C\in T_x\M$ determined by the conditioned variables. If $C=0$, the only solution away from the cut locus is $X_j=x$. If $C\ne0$, write $X_j=\exp_x(ru)$ before its directional cut distance. Then $u=C/\norm C$ and
\[
 r e^{-\beta r^2/2}=\norm C.
\]
The derivative of the left side is $e^{-\beta r^2/2}(1-\beta r^2)$, which is positive for $r<\beta^{-1/2}$ and negative for $r>\beta^{-1/2}$. The equation has at most two solutions. Thus the conditional solution set, including the cut locus, has volume zero. Absolute continuity implies conditional probability zero. Integrating the conditioning and taking the union over the finitely many $i$ proves the proposition.
\end{proof}

\section{Active sets, support boundaries, and exact enumeration}\label{app:active}
For the regions $U_A,C_A$ in~\eqref{eq:activeregions}, define the untruncated active objectives
\[
 F_A=\sum_{i\in A}q_i,\qquad G_A=\sum_{i\in A}a_i(1-\beta q_i).
\]
The proof applies under the component concavity in the main text. Assumption~\ref{ass:convex} supplies an explicit sufficient condition for it.

\begin{lemma}[No supported boundary maximum]\label{lem:boundary}
For either $S_T$, a point with $S_T(x)>0$ on one or more support boundaries is not a local maximum of $S_T$.
\end{lemma}
\begin{proof}
Let $B=\{j:d(x,\xi_j)=h\}$, which is nonempty. Near $x$, write the score as
\[
 S_T=F+\sum_{j\in B}a_j(u_j)_+,\qquad u_j=1-\beta q_j,
\]
where $F$ is the smooth contribution from strictly active kernels and $\epsilon$. For $T=\geo$, take $a_j=1$. The coefficients are smooth and positive near the boundary because $h<\inj(\M)$.

For a tangent vector $w$, the one-sided directional derivative along $\exp_x(tw)$, $t\downarrow0$, is
\[
 DS_T(x;w)=dF_x(w)+\sum_{j\in B}a_j(x)\max\{du_j(w),0\}.
\]
Adding the derivatives in opposite directions yields
\[
 DS_T(x;w)+DS_T(x;-w)=\sum_{j\in B}a_j(x)|du_j(w)|.
\]
Each $du_j$ is nonzero since $\norm{dq_j(x)}=h>0$. Choose $w$ on which one $du_j$ is nonzero. The sum of the two directional derivatives is then strictly positive. Both would have to be nonpositive at a local maximum, giving a contradiction. Since $-\log$ is strictly decreasing, a local energy minimum in the finite-energy domain would be a score maximum, and is therefore also excluded.
\end{proof}

\begin{proof}[Proof of Theorem~\ref{thm:active}]
On $C_A$, $S_\geo=\epsilon+|A|-\beta F_A$. Consequently
\[
 \grad F_A=-\sum_{i\in A}\log_x\xi_i,\qquad
 \Hess F_A\succeq |A|c g.
\]
The critical-point equation is~\eqref{eq:mean}. At a solution, Equation~\eqref{eq:loghessian} gives
$\Hess E_\geo=\Hess F_A/S_\geo\succ0$.

The set $U_A$ is strongly geodesically convex: the unique minimizing geodesic between two of its points lies in every ball defining the intersection. Suppose two distinct points $x,y\in U_A$ are critical for $F_A$. Let $\gamma:[0,1]\to U_A$ be that geodesic with constant speed $\ell=d(x,y)>0$. Then
\[
 (F_A\circ\gamma)''(t)=\Hess F_A(\dot\gamma,\dot\gamma)\ge |A|c\ell^2.
\]
Integrating from zero to one gives
$(F_A\circ\gamma)'(1)-(F_A\circ\gamma)'(0)\ge |A|c\ell^2>0$.
Both derivatives are zero by stationarity at the endpoints. This contradiction proves uniqueness.

For the corrected energy, a single active term $k_i=a_i(1-\beta q_i)$ has Hessian
\begin{align*}
 \Hess k_i={}&-\beta a_i\Hess q_i
 -\beta(da_i\otimes dq_i+dq_i\otimes da_i)\\
 &+(1-\beta q_i)\Hess a_i.
\end{align*}
If $w$ is a unit tangent vector in $B_h(\xi_i)$, then $\norm{dq_i}=d(x,\xi_i)<h$ and $0<1-\beta q_i\le1$. Hence
\[
 \Hess k_i(w,w)\le-\beta a_*c+2\beta Lh+H=-\mu.
\]
Under the main-text component concavity this same bound is an assumption; the calculation above shows how Assumption~\ref{ass:convex} implies it. Summing gives $\Hess G_A\preceq-|A|\mu g$. The preceding integrated-geodesic argument with the inequality reversed proves that $G_A$ has at most one critical point in $U_A$. At such a point in $C_A$, the score is positive and~\eqref{eq:loghessian} makes the energy Hessian positive definite.

Since $\grad a_i=-a_i\grad\log\theta_{\xi_i}$,
\[
 \grad k_i=\beta a_i\log_x\xi_i
 -a_i(1-\beta q_i)\grad\log\theta_{\xi_i}.
\]
Setting the sum equal to zero proves~\eqref{eq:correctedmean}. Lemma~\ref{lem:boundary} excludes support boundaries. Outside all supports the score is either zero, outside the finite-energy domain, or constant $\epsilon$, with zero Hessian. Therefore every memory belongs to a nonempty active set and every such set contributes at most one. Every critical point in a nonempty $C_A$ has just been shown to have positive-definite energy Hessian; thus there is no saddle in the smooth supported domain. This proves~\eqref{eq:subsetbound}. If all $N$ original patterns are memories, removing those $N$ distinct points leaves at most $2^N-N-1$ additional points, and imposing novelty can only reduce that number.
\end{proof}

Let $z_A^T$ denote the unique candidate in $U_A$ when it exists, with an absent candidate contributing zero. The theorem gives the exact deterministic formula
\begin{equation}
 N_\emg^T(\eta)=\sum_{\varnothing\ne A\subseteq[N]}
 \ind\{z_A^T\text{ exists},\ z_A^T\in C_A,\ \min_i d(z_A^T,\xi_i)\ge\eta\}.
 \label{eq:exactcount}
\end{equation}

\begin{proposition}[Finite-sample expected count and occupancy bound]\label{prop:finitecount}
Assume the geometric hypotheses of Theorem~\ref{thm:active} hold uniformly for the possible centers. Let the patterns be independent from an absolutely continuous law $\nu$, and let $0<\eta<h$. If $z_k^T$ is the candidate from $X_1,\ldots,X_k$ in their common support intersection, then
\begin{align}
 \E N_\emg^T(\eta)=\sum_{k=2}^N\binom Nk
 \E\big[&\ind\{z_k^T\text{ exists},\ \min_{i\le k}d(z_k^T,X_i)\ge\eta\}\nonumber\\
 &\hspace{1em}\cdot(1-b_h(z_k^T))^{N-k}\big].\label{eq:expectedcount}
\end{align}
Put $L_h=\sup_x\#\{i:d(x,X_i)<h\}$ and $p_{2h}=\sup_x\nu(B_{2h}(x))$. For $1\le k\le N-1$,
\begin{equation}
 \P(L_h>k)\le N\binom{N-1}{k}p_{2h}^k
 \le N\left(\frac{e(N-1)p_{2h}}k\right)^k.\label{eq:occupancy}
\end{equation}
On $\{L_h\le k\}$, $N_\tot^T\le\sum_{j=1}^k\binom Nj$.
\end{proposition}
\begin{proof}
Whenever the candidate exists inside $U_A$, its defining equation has a nonsingular derivative with respect to the candidate, by strict convexity or concavity. The implicit function theorem makes the candidate locally smooth in the pattern positions. Its existence domain is open, and uniqueness makes these local functions agree. The candidate is therefore a measurable function of the patterns on its existence domain.

For an active subset $A$ of size $k$, condition on its $k$ patterns. The candidate is then fixed whenever it exists. All other independent patterns must lie outside its radius-$h$ ball, which has conditional probability $(1-b_h(z_A^T))^{N-k}$. Since $\eta<h$, those excluded patterns automatically satisfy the novelty condition. Absolute continuity, when needed, removes equality on the sphere; alternatively use closed-ball exclusion throughout if the distribution has sphere atoms. Summing~\eqref{eq:exactcount} and using exchangeability proves~\eqref{eq:expectedcount} for absolutely continuous $\nu$. More generally the displayed formula is valid when $\nu(\partial B_h(z_k^T))=0$ almost surely. Singleton candidates are their original patterns, so they do not contribute.

If a radius-$h$ ball contains at least $k+1$ patterns, choose one of them. The other $k$ are all within distance $2h$ of the chosen pattern by the triangle inequality. For each possible chosen index, condition on its value and apply a union bound over subsets of $k$ other indices. The probability is at most $\binom{N-1}{k}p_{2h}^k$. A union bound over the $N$ chosen indices proves the first inequality in~\eqref{eq:occupancy}.

The bound $\binom{n}{k}\le(en/k)^k$ follows from $\binom nk\le n^k/k!$ and
\[
 \log(k!)=\sum_{j=1}^k\log j\ge\int_1^k\log t\,dt
 =k\log k-k+1\ge k\log k-k.
\]
This proves the second inequality. On $\{L_h\le k\}$, every realized nonempty active subset has at most $k$ members. The one-candidate-per-subset bound proves the final statement.
\end{proof}

For example, the integer condition
$k\ge\max\{e^2(N-1)p_{2h},\log(N/\delta)\}$ makes the right side of~\eqref{eq:occupancy} at most $\delta$, unless $k\ge N$, when the deterministic total bound suffices. An expected occupancy at a fixed deterministic $x$ does not itself control the occupancy at data-selected candidate locations.

\section{Quantitative persistence of the simplex construction}\label{app:simplex}
The Euclidean active-subset viewpoint is the starting point of \citet{hoover}. The result proved here controls its maximal-count simplex realization under curvature, including the volume correction. Averaging each active score is useful: it keeps the perturbation bounds independent of the subset size.

\begin{proof}[Proof of Theorem~\ref{thm:emergence}, designed-pattern part]
Fix an orthonormal coordinate system on $T_p\M$. All constants below are taken on one fixed normal neighborhood of $p$; they may depend on the metric and dimension but not on $N\le m+1$ or on the subset.

\paragraph{The Euclidean support margins.}
Let $e_1,\ldots,e_N$ be the standard basis of $\R^N$, let $\bar e=N^{-1}\sum_i e_i$, and put $w_i=e_i-\bar e$. The $w_i$ lie in the $(N-1)$-dimensional subspace orthogonal to $(1,\ldots,1)$. Embed that subspace isometrically in $T_p\M$. For a nonempty $A\subseteq[N]$ with $k=|A|$, write $\bar w_A=k^{-1}\sum_{i\in A}w_i$. Then $\norm{w_i}\le1$, $\norm{\bar w_A}\le1$, and
\begin{equation}
 \norm{\bar w_A-w_i}^2=
 \begin{cases}1-1/k,&i\in A,\\1+1/k,&i\notin A.\end{cases}
 \label{eq:simplexdistance}
\end{equation}
Indeed, $\bar w_A-w_i=k^{-1}\sum_{j\in A}e_j-e_i$. The squared norm of the first term is $k/k^2=1/k$; the cross term is $-2/k$ for $i\in A$ and zero otherwise. Adding $\norm{e_i}^2=1$ proves~\eqref{eq:simplexdistance}. Likewise $\norm{w_i-w_j}^2=2$ for $i\ne j$. The squared-distance margin from the support boundary is at least $1/N$ for every subset.

\paragraph{Uniform geometric perturbations.}
Define $\xi_i(h)=\exp_p(hw_i)$, and for $\norm y\le2$, $\norm w\le1$ put
\[
 D_h(y,w)=h^{-2}d(\exp_p(hy),\exp_p(hw))^2,
 \qquad a_h(y,w)=\theta_{\exp_p(hw)}(\exp_p(hy))^{-1}.
\]
There are $h_0>0$ and $C_0<\infty$, independent of $y,w,N$, such that
\begin{align}
 \norm{D_h(\cdot,w)-\norm{\cdot-w}^2}_{C^2(\{\norm y\le2\})}&\le C_0h^2,\label{eq:rescaleddistance}\\
 \norm{a_h(\cdot,w)-1}_{C^2(\{\norm y\le2\})}&\le C_0h^2,\qquad 0<h<h_0.\label{eq:rescaledjacobian}
\end{align}
We justify both bounds, including their derivatives. In normal coordinates at $p$, $g(0)=I$ and its first derivatives vanish. On a fixed rescaled ball, $g_h(z):=g(hz)$ consequently satisfies $g_h=I+O(h^2)$ with all fixed finite orders of $z$-derivatives. Its Christoffel symbols are $\Gamma_h(z)=h\Gamma(hz)=O(h^2)$, with the same derivative control. Choose the coordinate domain to contain $\norm z\le8$, and take $h_0$ small enough that its physical image is normal.

For the geodesic equation with initial position $y$ and initial velocity $u$, integration of $\ddot\gamma=-\Gamma_h(\gamma)(\dot\gamma,\dot\gamma)$ on $[0,1]$ gives
$\gamma(t)=y+tu+O(h^2)$ and $\dot\gamma(t)=u+O(h^2)$, uniformly for the bounded initial positions and velocities under consideration. Differentiating this differential equation with respect to $y,u$ gives linear differential equations whose coefficients and inhomogeneous perturbations are $O(h^2)$; their integral forms give the same estimates with two endpoint derivatives. For example, the position derivative $J$ satisfies $J(0)=I$, $\dot J(0)=0$, and
\[
 \ddot J=-D\Gamma_h(\gamma)[J](\dot\gamma,\dot\gamma)
 -2\Gamma_h(\gamma)(\dot J,\dot\gamma).
\]
Boundedness of $\gamma,\dot\gamma$ and integration imply $J=I+O(h^2)$ and $\dot J=O(h^2)$. Velocity derivatives have initial data $0,I$ and give $tI+O(h^2)$; differentiating once more gives the second-derivative estimates by the same integral bounds.

The endpoint map is therefore $y+u+O(h^2)$ in $C^2$, with derivative in $u$ equal to $I+O(h^2)$. The inverse function theorem, uniformly on the bounded endpoint domain, solves the endpoint equation with $u=w-y+O(h^2)$, including two derivatives in $y,w$. Equivalently, this uniform inverse follows by writing $u=w-y-R_h(y,u)$; the remainder has Lipschitz constant $O(h^2)$ in $u$ and maps a fixed small neighborhood of $w-y$ to itself. The resulting geodesic stays in the chosen domain and is the unique short minimizing geodesic when $h_0$ is small. Its rescaled energy is
\[
 D_h(y,w)=\int_0^1g_h(\gamma(t))(\dot\gamma(t),\dot\gamma(t))\,dt
 =\norm{y-w}^2+O(h^2)
\]
with two endpoint derivatives. This proves~\eqref{eq:rescaleddistance}.

For~\eqref{eq:rescaledjacobian}, the reciprocal volume density is smooth in both endpoints near the diagonal. It equals one on the diagonal and its first derivatives there vanish by Lemma~\ref{lem:jets} and differentiation of its diagonal value. Its difference from one is thus $O(h^2)$ for the two endpoints above. Its first endpoint derivative is $O(h)$; the chain rule supplies another factor $h$ in $y$ coordinates. Its bounded second endpoint derivative receives a factor $h^2$. This gives the claimed $C^2_y$ bound.

\paragraph{A critical point for every subset.}
On the rescaled domain define the averaged, untruncated active scores
\[
 H_{A,h}^\geo(y)=\frac1k\sum_{i\in A}(1-D_h(y,w_i)),\qquad
 H_{A,h}^\vc(y)=\frac1k\sum_{i\in A}a_h(y,w_i)(1-D_h(y,w_i)).
\]
Their Euclidean limit is
$H_{A,0}(y)=k^{-1}\sum_{i\in A}(1-\norm{y-w_i}^2)$.
Since $\norm y\le2$ and $\norm{w_i}\le1$, products and derivatives in these expressions are uniformly bounded. Increasing $C_0$ to a constant $C\ge1$ gives, for both models,
\begin{equation}
 \norm{H_{A,h}^T-H_{A,0}}_{C^2}\le Ch^2,
 \qquad
 \nabla H_{A,0}(y)=-2(y-\bar w_A),\quad
 \nabla^2H_{A,0}=-2I.
 \label{eq:averagedperturbation}
\end{equation}
The averaging ensures that $C$ is independent of $k$.

Set $\rho_N=1/(32N)$ and assume
\begin{equation}
 Ch^2\le1/(64N).
 \label{eq:margincondition}
\end{equation}
On the boundary of the Euclidean closed ball $\overline B_{\rho_N}(\bar w_A)$, the outward unit derivative of $H_{A,h}^T$ is at most
$-2\rho_N+Ch^2\le-3\rho_N/2<0$.
The continuous score attains a maximum on this compact ball. A boundary maximizer is impossible, since moving a short distance inward would increase the score. Thus an interior maximizer $y_A^T$ exists and has zero gradient. Equation~\eqref{eq:averagedperturbation} gives $\nabla^2H_{A,h}^T\preceq-I$, so the maximizer is unique in the ball. At its critical point,
\[
 2\norm{y_A^T-\bar w_A}
 =\norm{\nabla H_{A,0}(y_A^T)-\nabla H_{A,h}^T(y_A^T)}
 \le Ch^2.
\]
Hence $\norm{y_A^T-\bar w_A}\le Ch^2/2$.

\paragraph{Checking every active set and novelty margin.}
For $\delta_A=y_A^T-\bar w_A$, the squared Euclidean distance changes by
\[
 \left|\norm{y_A^T-w_i}^2-\norm{\bar w_A-w_i}^2\right|
 \le(2\norm{\bar w_A-w_i}+\norm{\delta_A})\norm{\delta_A}
 \le5\norm{\delta_A}.
\]
The last inequality uses $\norm{\bar w_A},\norm{w_i}\le1$ and $\norm{\delta_A}\le1$. Combining it with~\eqref{eq:rescaleddistance} and enlarging the common $C$ beforehand if necessary gives
\[
 \left|D_h(y_A^T,w_i)-\norm{\bar w_A-w_i}^2\right|
 \le Ch^2+\tfrac52Ch^2=\tfrac72Ch^2<\frac1{4N}.
\]
By~\eqref{eq:simplexdistance}, all members are strictly within radius $h$ and all nonmembers strictly outside. Thus $x_A^T=\exp_p(hy_A^T)$ lies in $C_A$. Different subsets give different points because a point has only one strict active set. The coordinate Hessian of the energy is positive definite at $x_A^T$: its score Hessian is a positive multiple of $\nabla^2H_{A,h}^T\prec0$ and its gradient is zero. At a critical point the connection term in the intrinsic Hessian vanishes, so this is also positive definiteness of the Riemannian Hessian.

For $k\ge2$, member distances satisfy
\[
 h^{-2}d(x_A^T,\xi_i)^2\ge1-1/k-\tfrac72Ch^2
 \ge\tfrac12-\tfrac72Ch^2.
\]
Fix $\tau<1/\sqrt2$ and put $g_\tau=1/2-\tau^2>0$. Choose, for example,
\[
 c_\tau=\min\left\{h_0/2,\ (64C)^{-1/2},\ (g_\tau/(8C))^{1/2}\right\},
\]
and shrink $h_0$ beforehand so Lemma~\ref{lem:small} holds. If $h\le c_\tau/\sqrt N$ with $N\ge2$, then~\eqref{eq:margincondition} holds and $(7/2)Ch^2<g_\tau$. Hence all member distances exceed $\tau h$. Nonmember distances exceed $h>\tau h$.

Originals are isolated because $h^{-2}d(\xi_i,\xi_j)^2=2+O(h^2)>1$. At this small bandwidth they are stable for both models by Theorem~\ref{thm:isolated}. Each is therefore the unique singleton candidate. Theorem~\ref{thm:active} applies by Lemma~\ref{lem:small} and bounds the total number by $2^N-1$. We have constructed that many distinct memories, so the count is exact and exactly $2^N-N-1$ of them meet the stated novelty threshold. Finally, the differential of $y\mapsto\exp_p(hy)$ has norm at most $2h$ on the bounded domain for small $h$. Integrating along the segment from $\bar w_A$ to $y_A^T$ gives
\[
 d(x_A^T,\exp_p(h\bar w_A))\le2h\norm{y_A^T-\bar w_A}\le Ch^3.
\]
This proves the displacement claim and completes the quantitative construction.
\end{proof}

\begin{corollary}[Polynomial bandwidths for the spherical construction]\label{cor:quantitative-sphere}
Fix $0<\tau<1/\sqrt2$. There is a constant $c_\tau>0$, independent of $m,N,p$, such that on the unit sphere $S^m$, $m\ge2$, the construction in Theorem~\ref{thm:emergence} attains~\eqref{eq:maximal} whenever
\[
 h\le c_\tau/\sqrt N\quad\text{for }E_\geo,
 \qquad
 h\le c_\tau/\sqrt{mN}\quad\text{for }E_\vc,
 \qquad 2\le N\le m+1.
\]
These are sufficient, not necessary, bandwidth conditions. In particular, for $N=m+1$ the sufficient shrinkage is polynomial in dimension for both energies.
\end{corollary}
\begin{proof}
We verify dimension-uniform versions of the perturbation and concavity bounds used above. Sphere symmetry makes all constants independent of the center $p$.

In normal coordinates on the unit sphere, for $r=\norm v$ and a tangent vector $u$ the metric is
\[
 g_v(u,u)=\left(\frac{\sin r}{r}\right)^2\norm u^2
 +\left[1-\left(\frac{\sin r}{r}\right)^2\right]\frac{\ip{v}{u}^2}{r^2},
\]
with the continuous extension at $r=0$. To see this, differentiate
$\exp_pv=\cos(r)p+(\sin r/r)v$ in the ambient Euclidean space, decomposing $u$ into its component parallel to $v$ and its orthogonal component. The parallel component keeps its length; the orthogonal component is multiplied by $\sin r/r$, and the images are orthogonal. Squaring their lengths gives the formula.

Put $A(s)=(\sin\sqrt s/\sqrt s)^2$ and $B(s)=(1-A(s))/s$. Their power series have finite smooth extensions near zero. The metric is therefore the operator
$g_v=A(\norm v^2)I+B(\norm v^2)v\otimes v$.
For $v=hz$ on a fixed bounded rescaled ball, the differences $g_{hz}-I$ and their derivatives up to any fixed order in $z$ are $O(h^2)$ in operator norm, with constants independent of $m$. This follows from the bounded one-variable derivatives of $A,B$, $A(0)=1$, and the dimension-free bounds on derivatives of $z\mapsto\norm z^2$ and $z\mapsto z\otimes z$. The inverse metric is uniformly bounded for small $h$. The Christoffel symbols, viewed as a bilinear operator, and their required derivatives are consequently $O(h^2)$ with dimension-uniform constants, by the metric formula for the Levi-Civita connection. The integrated geodesic and endpoint-inversion estimates in~\eqref{eq:rescaleddistance} then have a constant independent of dimension. It follows that the geodesic averaged-score error in~\eqref{eq:averagedperturbation} is at most $Ch^2$, with a universal $C$.

The spherical reciprocal volume factor is
\[
 a_h(y,w)=\exp\!\left((m-1)\ell(h^2D_h(y,w))\right),
 \qquad \ell(s)=-\log\left(\frac{\sin\sqrt s}{\sqrt s}\right).
\]
The function $\ell$ is smooth near zero, $\ell(0)=0$, and $\ell(s)=s/6+O(s^2)$. The preceding uniform $C^2$ bounds on $D_h$ imply that the exponent and its first two $y$-derivatives have magnitude at most $Cmh^2$. For example, its gradient is $(m-1)\ell'(h^2D_h)h^2\nabla D_h$; its Hessian is
\[
 (m-1)\left[\ell'(h^2D_h)h^2\nabla^2D_h
 +\ell''(h^2D_h)h^4\nabla D_h\otimes\nabla D_h\right].
\]
If $mh^2$ is smaller than a fixed constant, exponentiation gives
$\norm{a_h-1}_{C^2}\le C'mh^2$: its Hessian is the exponential times the sum of the exponent Hessian and the outer product of its gradient, whose norm is $O(mh^2+(mh^2)^2)$. Thus the corrected averaged-score error is at most $C''mh^2$.

We also need uniform component concavity to exclude additional candidates outside the constructed neighborhoods. The Hessian of $q_p=d(\cdot,p)^2/2$ has eigenvalue $1$ in the radial direction and $r\cot r$ in orthogonal directions. This follows by differentiating radial distance in the polar metric $dr^2+\sin^2r\,g_{S^{m-1}}$: for a unit tangential vector $u$, $\Hess r(u,u)=\cot r$, and $\Hess q=dr\otimes dr+r\Hess r$. For a fixed small radius these eigenvalues are at least $1/2$, uniformly in $m$. The argument in Lemma~\ref{lem:small} therefore gives a dimension-uniform strongly convex support radius.

For the corrected kernel, write $a(r)=\exp((m-1)\ell(r^2))$. For $r\le h$ and $mh^2$ bounded by a fixed small constant, $1\le a(r)\le C$. Direct differentiation gives
\[
 a'(r)=2(m-1)r\ell'(r^2)a(r),
\]
\[
 a''(r)=a(r)\left[(m-1)(2\ell'(r^2)+4r^2\ell''(r^2))
 +4(m-1)^2r^2\ell'(r^2)^2\right].
\]
Consequently $|a'(r)|\le C_1mr$ and $|a''(r)|\le C_2(m+m^2r^2)\le C_3m$. The radial Hessian formula gives the other eigenvalue $a'(r)\cot r$, also bounded by $C_1m$ because $r\cot r\le1$ at these radii. Assumption~\ref{ass:convex} therefore holds with $a_*\ge1$, $c\ge1/2$, $L\le C_1mh$, $H\le C_4m$, provided
\[
 \mu\ge h^{-2}-(4C_1+C_4)m>0.
\]
It suffices to bound $mh^2$ by a sufficiently small dimension-independent constant.

Now repeat the preceding simplex proof with error parameter $Ch^2$ for the geodesic score and $Cmh^2$ for the corrected score. The critical-point, support, and novelty arguments require this parameter to be at most $1/(64N)$ and a fixed fraction of $1/2-\tau^2$. Choose $c_\tau$ small enough to ensure these requirements and the uniform support-convexity and corrected-concavity bounds just proved. The two displayed bandwidth conditions then meet every requirement. The construction has all $2^N-1$ memories and the required novelty count, completing the proof.
\end{proof}

\section{A dimension-sensitive counting bound on spheres}\label{app:spherecount}
We use the hyperplane-arrangement region bound associated with \citet{zaslavsky}; the recurrence is included below for completeness.

\begin{proposition}[Polynomial upper bound at fixed spherical dimension]
For patterns on the unit sphere $S^m$, suppose Assumption~\ref{ass:convex} holds and $h<\pi/2$. Then for both energies,
\[
 N_\tot^T\le\sum_{j=1}^{m+1}\binom Nj.
\]
If every original is retained, then
$N_\emg^T(\eta)\le\sum_{j=1}^{m+1}\binom Nj-N$.
We use the convention $\binom Nj=0$ for $j>N$.
\end{proposition}
\begin{proof}
For unit vectors $x$ and $\xi_i$, geodesic activation is equivalent to
\[
 d(x,\xi_i)<h\quad\Longleftrightarrow\quad x\cdot\xi_i>\cos h.
\]
Every strict active pattern on the sphere is consequently a strict sign pattern of the $N$ affine hyperplanes $z\cdot\xi_i=\cos h$ in $\R^{m+1}$. For a fixed strict sign pattern, its realization in the ambient space is an intersection of open halfspaces, hence is convex and, when nonempty, connected. Thus distinct strict sign patterns correspond to distinct regions of the hyperplane arrangement.

Let $R(N,d)$ be the maximal number of regions cut out by $N$ affine hyperplanes in $\R^d$. Adding the last hyperplane can split an old region only when its intersection with that hyperplane contains a relatively open set. The preceding hyperplanes divide the new hyperplane into at most $R(N-1,d-1)$ such regions, and each of these can lie in and split at most one old region. Therefore
\[
 R(N,d)\le R(N-1,d)+R(N-1,d-1).
\]
The initial conditions are $R(0,d)=1$ and $R(N,0)=1$. Induction using Pascal's identity gives
\[
 R(N,d)\le\sum_{j=0}^d\binom Nj.
\]
Because $\cos h>0$, the ambient point $z=0$ strictly satisfies all inactive inequalities. The empty sign pattern therefore occupies at least one ambient region, whether or not it occurs on the sphere. The number of nonempty spherical active patterns is at most $R(N,m+1)-1$. Theorem~\ref{thm:active} contributes at most one memory per nonempty active pattern, proving the total bound. Removing the $N$ original memories proves the additional-memory bound.
\end{proof}

\section{Random storage and small-ball probabilities}\label{app:capacity}

\begin{theorem}[Collision characterization and sharp random capacity]\label{thm:capacity}
At a fixed deterministic bandwidth $h<\inj(\M)$, for any absolutely continuous $\nu$,
\[
 \mathcal A_\geo=\{d(X_i,X_j)>h\text{ for every }i\ne j\}\quad\text{almost surely}.
\]
The same statement holds for $\mathcal A_\vc$ for all sufficiently small $h$ on a fixed smooth compact manifold: smallness is used for the Ricci stability condition and for invertibility of the neighbor-to-gradient derivative in Lemma~\ref{lem:generic}. In particular,
\begin{equation}
 \P(\mathcal A_T)\ge1-\binom N2q_h.\label{eq:union}
\end{equation}
Suppose $q_h\to0$ and $\sup_x b_h(x)\le Cq_h$, where $C$ is independent of $h$ on a fixed manifold or of the dimension $m$ along the sphere sequences considered below. If
$N^2q_h\to2\lambda\in(0,\infty)$, then
\begin{equation}
 \P(\mathcal A_T)\longrightarrow e^{-\lambda}.\label{eq:storagepoisson}
\end{equation}
Let $N_\delta^T(h)$ be the largest $N$ with success probability at least $1-\delta$, for fixed $0<\delta<1$. Then
\begin{equation}
 N_\delta^T(h)\sim\sqrt{\frac{-2\log(1-\delta)}{q_h}}.\label{eq:sharpN}
\end{equation}
For a continuous density $f$ on a fixed compact $m$-dimensional manifold,
\[
 q_h=\omega_mh^m\left(\int_\M f^2\,dV+o(1)\right),
\]
so both energies have the asymptotic all-pattern capacity
\begin{equation}
 N_\delta^T(\beta)\sim
 \left(\frac{-2\log(1-\delta)}{\omega_m\int f^2\,dV}\right)^{1/2}
 \left(\frac\beta2\right)^{m/4}.\label{eq:capacitybeta}
\end{equation}
\end{theorem}

\begin{lemma}[Generic nonstationarity with an active neighbor]\label{lem:generic}
At a fixed deterministic $h<\inj(\M)$, independent absolutely continuous patterns almost surely have no stationary original with an active neighbor for $E_\geo$. The same holds for $E_\vc$ for every sufficiently small $h$ on a fixed smooth compact manifold.
\end{lemma}
\begin{proof}
Distance-equality events $d(X_i,X_j)=h$ have probability zero: conditional on $X_i=x$, the sphere of radius $h$ lies in a normal chart and has volume zero. Fix $i\ne j$, condition on $X_i=x$ and all patterns except $X_j$, and restrict to $d(x,X_j)<h$.

For the geodesic energy, the own contribution to the score gradient is zero and stationarity requires
$\beta\log_xX_j=C$ for a fixed tangent vector $C$. The exponential map is one-to-one on the radius-$h$ tangent ball, so the conditional equation has at most one solution and probability zero. A finite union over $(i,j)$ proves the geodesic assertion.

For the corrected energy, write $a(x,y)=\theta_y(x)^{-1}$ and $q(x,y)=d(x,y)^2/2$. For fixed $x$, the contribution of $y$ to the score gradient is
\[
 F_x(y)=(1-\beta q(x,y))\grad_xa(x,y)-\beta a(x,y)\grad_xq(x,y).
\]
Differentiating with respect to $y$ and dividing by $\beta$ gives
\begin{align*}
 \beta^{-1}D_yF_x={}&-aD_y\grad_xq-(d_ya)\otimes\grad_xq
 -(d_yq)\otimes\grad_xa\\
 &+(\beta^{-1}-q)D_y\grad_xa.
\end{align*}
Use parallel transport to identify the nearby tangent spaces. On the diagonal, $D_y\grad_xq=-I$, and this derivative converges uniformly to $-I$ as the endpoints approach one another. Also $a=1+O(h^2)$, both first derivatives of $a$ are $O(h)$, both first derivatives of $q$ are $O(h)$, and $D_y\grad_xa$ is uniformly bounded. Finally, $|\beta^{-1}-q|\le h^2/2$ inside the support. Thus $\beta^{-1}D_yF_x$ converges uniformly to $I$ for $d(x,y)<h$, and is invertible once $h$ is sufficiently small.

By the inverse function theorem, $F_x$ is locally one-to-one everywhere on $B_h(x)$. Each fiber is therefore a discrete subset of that ball. A discrete subset of a second-countable manifold is countable: assign to each point a basis neighborhood meeting the fiber in that point only, and use the resulting injection into a countable basis. The fiber consequently has volume zero. Stationarity under the conditioning requires $F_x(X_j)=C$ for a fixed vector. Its conditional probability is zero by absolute continuity. Integrating and taking the finite union over $(i,j)$ proves the corrected assertion.
\end{proof}

\begin{lemma}[No-collision limit]\label{lem:collision}
Suppose $q_h\to0$, $\sup_xb_h(x)\le Cq_h$, and $N^2q_h\to2\lambda\in(0,\infty)$. If
\[
 Z_h=\sum_{i<j}\ind\{d(X_i,X_j)<h\},
\]
then $\P(Z_h=0)\to e^{-\lambda}$. If $N\to\infty$, $q_h\to0$, and $N^2q_h\to\infty$, then $\P(Z_h=0)\to0$.
\end{lemma}
\begin{proof}
Fix $k\ge1$. Expand $\E\binom{Z_h}{k}$ over sets of $k$ distinct unordered edges on the sample indices. Edges with disjoint endpoints contribute
\[
 \frac{(N)_{2k}}{2^kk!}q_h^k\longrightarrow\frac{\lambda^k}{k!},
\]
where $(N)_{2k}=N(N-1)\cdots(N-2k+1)$.

For any edge pattern with overlapping endpoints, let $v$ be its number of vertices and $c$ its number of connected components. Each component has at least two vertices, and at least one has three or more, so $v>2c$. Select a spanning forest. Reveal one root in each component and then every other vertex after its parent. Each required forest edge has conditional probability at most $Cq_h$, so the probability of the whole edge pattern is at most $(Cq_h)^{v-c}$. For fixed $k$ there are at most a constant times $N^v$ labelings of each of finitely many patterns. Since $N=O(q_h^{-1/2})$, their total contribution is
\[
 O_k(N^v q_h^{v-c})=O_k(q_h^{v/2-c})\longrightarrow0.
\]
We have proved $\E\binom{Z_h}{k}\to\lambda^k/k!$ for every fixed $k$.

For an integer $z\ge1$, Pascal's identity and cancellation give
\[
 \sum_{k=0}^r(-1)^k\binom zk=(-1)^r\binom{z-1}{r}.
\]
For $z=0$ the sum is one. Therefore the odd and even partial sums bound $\ind\{z=0\}$. Applying these bounds to $Z_h$ yields, for each fixed $K$,
\[
 \sum_{k=0}^{2K+1}(-1)^k\E\binom{Z_h}{k}
 \le\P(Z_h=0)\le
 \sum_{k=0}^{2K}(-1)^k\E\binom{Z_h}{k}.
\]
First let $h\to0$, using the moment limits, and then let $K\to\infty$. The two limiting series both tend to $e^{-\lambda}$.

For the second assertion, disjoint edges are independent. Two distinct edges sharing a vertex have joint probability
$\E[b_h(X_1)^2]\le Cq_h\E b_h(X_1)=Cq_h^2$. There are at most a constant times $N^3$ such edge pairs. Hence
\[
 \Var(Z_h)\le \E Z_h+C' N^3q_h^2,
 \qquad \E Z_h=\binom N2q_h.
\]
If $N^2q_h\to\infty$, then
\[
 \P(Z_h=0)\le\frac{\Var(Z_h)}{(\E Z_h)^2}
 \le\frac{C''}{N^2q_h}+\frac{C''}{N}\longrightarrow0,
\]
by Chebyshev's inequality.
\end{proof}

\begin{proof}[Proof of Theorem~\ref{thm:capacity}]
In the absence of a pair at distance at most $h$, Theorem~\ref{thm:isolated} stores every original for the geodesic energy and, uniformly on a fixed compact manifold, for the corrected energy once $h$ is small. Conversely, Lemma~\ref{lem:generic} shows that any active neighbor almost surely prevents stationarity at an original. Equality events have probability zero. This proves the event identities.

The union bound over $\binom N2$ possible close pairs proves~\eqref{eq:union}. Lemma~\ref{lem:collision} proves~\eqref{eq:storagepoisson}. To obtain~\eqref{eq:sharpN}, put $c_\delta=\sqrt{-2\log(1-\delta)}$. For any constants $0<c_-<c_\delta<c_+$, sample sizes $\lfloor c_-q_h^{-1/2}\rfloor$ and $\lceil c_+q_h^{-1/2}\rceil$ have success probabilities tending respectively to $e^{-c_-^2/2}>1-\delta$ and $e^{-c_+^2/2}<1-\delta$. The no-collision probability is nonincreasing in $N$. These two sample sizes therefore sandwich $N_\delta^T(h)$ for small $h$. Sending $c_-$ and $c_+$ to $c_\delta$ gives the asymptotic equivalence.

For a continuous density, normal coordinates yield
\[
 b_h(x)=h^m\int_{\norm u<1} f(\exp_x(hu))\theta_x(\exp_x(hu))\,du
 =\omega_mh^m(f(x)+o(1))
\]
uniformly in $x$, by uniform continuity and Lemma~\ref{lem:jets}. Integrating against $f(x)dV(x)$ gives the expansion for $q_h$. Since $\int f^2>0$ and $f$ is bounded, it also gives $\sup_xb_h(x)\le Cq_h$ for small $h$. Substitution of $h=(2/\beta)^{1/2}$ proves~\eqref{eq:capacitybeta}.
\end{proof}

\begin{proposition}[Typical-pattern retention and the fraction of failures]\label{prop:typical}
Work in a regime in which each original is a memory if and only if it has no neighbor within distance $h$, almost surely. This holds for $E_\geo$ at fixed $h<\inj(\M)$, for $E_\vc$ at sufficiently small $h$ on a fixed compact manifold, and for uniform spherical patterns above the corrected stability threshold as shown in Corollary~\ref{cor:sphere}. Define
\[
 R_N^T=\frac1N\sum_{i=1}^N\ind\{X_i\text{ is a memory for }E_T\},\qquad
 p_N^T(h)=\P(X_1\text{ is a memory for }E_T).
\]
For $0<\delta<1$, let $N_{\mathrm{typ},\delta}^T(h)$ be the largest $N$ for which $p_N^T(h)\ge1-\delta$.

\textbf{Exact probability and a fraction bound.} For every $N\ge1$,
\begin{equation}
 p_N^T(h)=\E R_N^T=\int(1-b_h(x))^{N-1}\,d\nu(x),
 \qquad \E(1-R_N^T)\le(N-1)q_h.
 \label{eq:typicalexact}
\end{equation}
For every $\eta>0$,
\begin{equation}
 \P(1-R_N^T>\eta)\le\frac{(N-1)q_h}{\eta}.
 \label{eq:fractionbound}
\end{equation}
In particular, $Nq_h\to0$ implies a vanishing fraction of failed originals in probability.

\textbf{Homogeneous ball probabilities.} If $b_h(x)=q_h$ for $\nu$-almost every $x$ and $0<q_h<1$, then
\begin{equation}
 p_N^T(h)=(1-q_h)^{N-1},\qquad
 N_{\mathrm{typ},\delta}^T(h)
 =1+\left\lfloor\frac{\log(1-\delta)}{\log(1-q_h)}\right\rfloor.
 \label{eq:typicalinteger}
\end{equation}
As $q_h\to0$,
\begin{equation}
 Nq_h\to\rho\in[0,\infty)\ \Longrightarrow\ p_N^T(h)\to e^{-\rho},\qquad
 N_{\mathrm{typ},\delta}^T(h)\sim\frac{-\log(1-\delta)}{q_h}.
 \label{eq:typicalhomogeneous}
\end{equation}

\textbf{A fixed nonuniform density.} Suppose $\M$ is fixed and compact, $f$ is a continuous density, and $h\to0$. Write $c_f=\int f^2\,dV>0$ and define
\begin{equation}
 G_f(\rho)=\int_\M f(x)\exp\!\left(-\rho\frac{f(x)}{c_f}\right)\,dV(x).
 \label{eq:typicalmixture}
\end{equation}
Then $Nq_h\to\rho\in[0,\infty)$ implies $p_N^T(h)\to G_f(\rho)$. There is a unique $\rho_\delta>0$ with $G_f(\rho_\delta)=1-\delta$, and
\begin{equation}
 N_{\mathrm{typ},\delta}^T(h)\sim\frac{\rho_\delta}{q_h}
 \sim\frac{\rho_\delta}{\omega_m c_f}\left(\frac\beta2\right)^{m/2}.
 \label{eq:typicalnonuniform}
\end{equation}
For homogeneous ball probabilities, and also for a fixed continuous density as above, $R_N^T\to1$ in probability if and only if $Nq_h\to0$ along a sequence with $q_h\to0$.
\end{proposition}
\begin{proof}
Condition on $X_1=x$. It is a memory exactly when each of the other $N-1$ independent patterns lies outside $B_h(x)$, an event of conditional probability $(1-b_h(x))^{N-1}$. Integrating over $x$ proves the first expression in~\eqref{eq:typicalexact}. Exchangeability gives $\E R_N^T=p_N^T(h)$. To see that this is also the success probability of a uniformly selected original, take an independent uniform index $I\in\{1,\ldots,N\}$. Then
\[
 \P(X_I\text{ is a memory})
 =\frac1N\sum_{i=1}^N\P(X_i\text{ is a memory})=p_N^T(h).
\]
For $0\le b\le1$, the event that at least one of $N-1$ independent Bernoulli trials succeeds has probability $1-(1-b)^{N-1}$ and is at most the sum $(N-1)b$ of the individual probabilities. Consequently
\[
 1-p_N^T(h)
 =\int[1-(1-b_h(x))^{N-1}]\,d\nu(x)
 \le(N-1)\int b_h(x)\,d\nu(x)=(N-1)q_h.
\]
Since $1-R_N^T\ge0$, its expectation is at least $\eta\P(1-R_N^T>\eta)$. This proves~\eqref{eq:fractionbound}. For each fixed $\eta>0$, its right-hand side tends to zero when $Nq_h\to0$, proving the stated convergence in probability. No independence of the retention indicators is used.

When $b_h=q_h$ almost everywhere, the integral in~\eqref{eq:typicalexact} equals $(1-q_h)^{N-1}$. As $\log(1-q_h)<0$, the success inequality is equivalent to
\[
 (N-1)\log(1-q_h)\ge\log(1-\delta)
 \quad\Longleftrightarrow\quad
 N-1\le\frac{\log(1-\delta)}{\log(1-q_h)}.
\]
Taking the largest integer $N$ proves~\eqref{eq:typicalinteger}. For $0\le u<1$, integration of the derivative of $-\log(1-u)$ gives
\begin{equation}
 0\le-\log(1-u)-u
 =\int_0^u\frac{t}{1-t}\,dt
 \le\frac{u^2}{2(1-u)}.
 \label{eq:logerror}
\end{equation}
Thus $\log(1-q_h)=-q_h+O(q_h^2)$ as $q_h\to0$. When $Nq_h\to\rho<\infty$, we have $(N-1)q_h\to\rho$ and $Nq_h^2\to0$. Multiplying the logarithm expansion by $N-1$ proves $(1-q_h)^{N-1}\to e^{-\rho}$. Dividing~\eqref{eq:logerror} by $u>0$ also proves $-\log(1-q_h)/q_h\to1$. The integer rounding in~\eqref{eq:typicalinteger} changes $N$ by at most a constant, whereas $q_h^{-1}\to\infty$, so the capacity equivalent in~\eqref{eq:typicalhomogeneous} follows.

For a fixed continuous density, the uniform small-ball expansion in the proof of Theorem~\ref{thm:capacity} gives
\[
 \frac{b_h(x)}{q_h}\longrightarrow\frac{f(x)}{c_f}
 \quad\text{uniformly in }x,\qquad
 \sup_x b_h(x)=O(q_h).
\]
If $Nq_h\to\rho<\infty$, inequality~\eqref{eq:logerror} yields
\[
 \sup_x\left|(N-1)\log(1-b_h(x))+(N-1)b_h(x)\right|
 \le\frac{N(\sup_x b_h(x))^2}{2(1-\sup_x b_h(x))}
 =O(Nq_h^2)\longrightarrow0.
\]
Meanwhile $(N-1)b_h(x)\to\rho f(x)/c_f$ uniformly. Exponentiating proves uniform convergence of $(1-b_h(x))^{N-1}$ to $\exp(-\rho f(x)/c_f)$. Integration against the probability measure $f\,dV$ establishes~\eqref{eq:typicalmixture} as the limit of $p_N^T(h)$.

The function $G_f$ is continuous by dominated convergence, since its integrand is bounded by the integrable function $f$. Also $G_f(0)=1$. The set where $f=0$ has $\nu$-measure zero because its integral against $f\,dV$ is zero. For $\rho_2>\rho_1\ge0$, the exponential factor strictly decreases at every point with $f(x)>0$, so $G_f(\rho_2)<G_f(\rho_1)$. Finally, this factor tends to zero as $\rho\to\infty$ at every such point; another application of dominated convergence gives $G_f(\rho)\to0$. The intermediate value theorem and strict decrease give a unique $\rho_\delta>0$ with $G_f(\rho_\delta)=1-\delta$.

Choose any $0<a<\rho_\delta<b$. For $n_a(h)=\lfloor a/q_h\rfloor$ and $n_b(h)=\lceil b/q_h\rceil$, the probability limits just proved give
\[
 p_{n_a(h)}^T(h)\longrightarrow G_f(a)>1-\delta,
 \qquad
 p_{n_b(h)}^T(h)\longrightarrow G_f(b)<1-\delta.
\]
For sufficiently small $h$, monotonicity of $(1-b_h(x))^{N-1}$ in $N$ therefore implies
$n_a(h)\le N_{\mathrm{typ},\delta}^T(h)<n_b(h)$.
Multiplying by $q_h$, taking lower and upper limits, and then letting $a\uparrow\rho_\delta$ and $b\downarrow\rho_\delta$ proves the first equivalent in~\eqref{eq:typicalnonuniform}. The expansion $q_h\sim\omega_m c_f h^m$ and $h=(2/\beta)^{1/2}$ prove the second.

It remains to prove necessity for the vanishing-failure assertion. If $R_N^T\to1$ in probability, then for every $0<\eta<1$,
\[
 0\le1-p_N^T(h)=\E(1-R_N^T)
 \le\eta+\P(1-R_N^T>\eta).
\]
Taking the upper limit and then $\eta\downarrow0$ shows $p_N^T(h)\to1$. Under homogeneous ball probabilities,
$p_N^T(h)=(1-q_h)^{N-1}\le e^{-(N-1)q_h}$.
If $Nq_h$ failed to tend to zero, a subsequence would have $Nq_h\ge a>0$. As $q_h\to0$, that subsequence eventually has $(N-1)q_h\ge a/2$, contradicting $p_N^T(h)\to1$.

For a fixed continuous density, take the same subsequence and set $n(h)=\lfloor a/(2q_h)\rfloor$. Eventually $n(h)\le N$, and monotonicity gives
\[
 p_N^T(h)\le p_{n(h)}^T(h)\longrightarrow G_f(a/2)<1,
\]
again a contradiction. Sufficiency was already proved using~\eqref{eq:fractionbound}. This completes all assertions.
\end{proof}

The distinction from all-pattern storage is quantitative: at homogeneous ball probabilities the fixed-target capacities have orders $q_h^{-1}$ and $q_h^{-1/2}$, respectively. A fixed nonuniform density has the same typical-capacity order, but its leading constant is determined by~\eqref{eq:typicalmixture}, not generally by a single exponential $e^{-Nq_h}$. These are exact-retention criteria; a corrupted-query success event additionally requires control of attraction neighborhoods.

\begin{corollary}[Unit spheres]\label{cor:sphere}
For uniform patterns on $S^m$, $m\ge2$, and $0<h<\pi$,
\[
 q_m(h)=\frac{\int_0^h\sin^{m-1}t\,dt}{\int_0^\pi\sin^{m-1}t\,dt}.
\]
At fixed $0<h<\pi/2$, the geodesic capacity satisfies
\[
 \log N_\delta^\geo(m,h)=-\tfrac m2\log(\sin h)+o(m),\qquad
 \log N_{\mathrm{typ},\delta}^\geo(m,h)=-m\log(\sin h)+o(m).
\]
For the corrected energy, at any fixed admissible $\beta=2/h^2$,
\[
 N_\delta^\vc(m,\beta)=
 \begin{cases}
 0,&\beta\le(m-1)/3,\\
 N_\delta^\geo(m,\beta),&\beta>(m-1)/3.
 \end{cases}
\]
The same dichotomy holds with $N_\delta^T$ replaced by $N_{\mathrm{typ},\delta}^T$.
For $h=c/\sqrt m$ with fixed $0<c<\sqrt6$, both energies satisfy
\[
 \log N_\delta^T(m,h)=\frac m4\log(m/c^2)+O(\log m),\qquad
 \log N_{\mathrm{typ},\delta}^T(m,h)=\frac m2\log(m/c^2)+O(\log m).
\]
Here a capacity of zero means that no positive sample size has the required probability of nondegenerate exact storage.
\end{corollary}

\begin{proof}[Proof of Corollary~\ref{cor:sphere}, fixed-bandwidth statements]
On the unit sphere, the Jacobi equation has solution $\sin r$ in each of the $m-1$ directions perpendicular to a radial geodesic and solution $r$ in the radial direction. Hence
\[
 \theta_p(x)=\left(\frac{\sin r}{r}\right)^{m-1},\qquad r=d(p,x).
\]
The polar volume element is proportional to $\sin^{m-1}r\,dr$ times angular volume. Integrating a spherical cap and dividing by the whole volume gives the displayed formula for $q_m(h)$. Homogeneity makes $b_h(x)=q_m(h)$ for every $x$.

Fix $0<h<\pi/2$. For $0<a<h$, the numerator is at least $a\sin^{m-1}(h-a)$ and at most $h\sin^{m-1}h$. For $0<b<\pi/2$, the denominator is at least $2b\cos^{m-1}b$ and at most $\pi$. Taking logarithms and dividing by $m$ therefore gives
\[
 \log\sin(h-a)\le\liminf_m m^{-1}\log q_m(h)
 \le\limsup_m m^{-1}\log q_m(h)
 \le\log\sin h-\log\cos b.
\]
Letting $a,b\downarrow0$ proves $m^{-1}\log q_m(h)\to\log\sin h$. The collision proof uses only the bound $b_h\le Cq_h$, with $C=1$ here, and thus applies along this dimension-indexed sequence as well. Equation~\eqref{eq:sharpN} gives the all-pattern exponential rate. The homogeneous formula~\eqref{eq:typicalhomogeneous} gives $\log N_{\mathrm{typ},\delta}^\geo=-\log q_m(h)+O(1)$, proving the typical-pattern rate.

The corrected isolated Hessian is given by~\eqref{eq:ricci} with $\Ric=(m-1)g$. It remains to exclude nonisolated stationary originals at an arbitrary admissible fixed $h$, rather than only at small $h$. A corrected component is radial with profile
\[
 k(r)=(1-\beta r^2/2)\left(\frac r{\sin r}\right)^{m-1},\qquad 0<r<h<\pi.
\]
At a fixed point $x$, its gradient contribution from $y=\exp_x(ru)$ is $-k'(r)u$. The function $k'$ is analytic on $(0,h)$, extends with $k'(0)=0$, and is not identically zero because its limit at $h$ is
$-\beta h(h/\sin h)^{m-1}\ne0$.

For a prescribed nonzero tangent vector, the equation $-k'(r)u=C$ restricts $u$ to the two directions parallel to $C$ and restricts $r$ to the zeros of $k'(r)\pm\norm C$. Neither analytic function is identically zero, so these radii form a discrete, hence countable, set. The preimage has volume zero. For $C=0$, the admissible radii are the discrete zeros of $k'$, giving a countable union of geodesic spheres and the center, again a null set. The conditioning proof from Lemma~\ref{lem:generic} therefore applies at every fixed admissible $h$. Every original with an active neighbor is almost surely nonstationary. An isolated original is nondegenerately stable exactly when $\beta>(m-1)/3$. Below or at this threshold no original can be a nondegenerate memory almost surely; above it the no-collision event is again exactly all-pattern storage. For a single specified original, the same argument gives retention probability zero below or at the threshold and $(1-q_m(h))^{N-1}$ above it. The all-pattern and typical-pattern capacity alternatives both follow.
\end{proof}

\begin{proof}[Proof of Corollary~\ref{cor:sphere}, shrinking-bandwidth statement]
Fix $0<c<\sqrt6$ and take $h=c/\sqrt m$. Then $\beta=2m/c^2>(m-1)/3$, so the two capacity functions agree by the fixed-bandwidth part, applied separately at each dimension. We compute the cap probability with constants uniform in $m$.

For $0\le t\le1$, Taylor's alternating bound gives $1-t^2/6\le\sin t/t\le1$. For $0\le z\le1/2$, integrating $1/(1-z)\le2$ gives $\log(1-z)\ge-2z$. Consequently
\[
 e^{-t^2/3}\le\frac{\sin t}{t}\le1.
\]
For sufficiently large $m$ we have $h\le1$. Raising to power $m-1$ and integrating gives
\[
 e^{-(m-1)h^2/3}\frac{h^m}{m}
 \le\int_0^h\sin^{m-1}t\,dt\le\frac{h^m}{m}.
\]
The exponential factor is bounded below by $e^{-c^2/3}$.
For the denominator, substitute $y=t-\pi/2$. Since $\tan y\ge y$ for $0\le y<\pi/2$ (its derivative is at least one), integration of $(\log\cos y)'=-\tan y$ gives $\cos y\le e^{-y^2/2}$. Symmetry therefore implies
\[
 \int_0^\pi\sin^{m-1}t\,dt
 \le\int_{\R}e^{-(m-1)y^2/2}\,dy
 =\sqrt{\frac{2\pi}{m-1}}.
\]
For $|y|\le m^{-1/2}$, $\cos y\ge1-y^2/2\ge1-1/(2m)$. The preceding logarithm bound yields $(1-1/(2m))^{m-1}\ge e^{-1}$. Hence
\[
 \int_0^\pi\sin^{m-1}t\,dt\ge2e^{-1}m^{-1/2}.
\]
Combining numerator and denominator bounds gives constants $c_1,c_2>0$, depending only on $c$, with
\[
 c_1\frac{h^m}{\sqrt m}\le q_m(h)\le c_2\frac{h^m}{\sqrt m}.
\]
Thus $\log q_m(h)=m\log h-\tfrac12\log m+O(1)$.
Uniform spherical sampling has $b_h(x)=q_m(h)$ at every point. The proof of Lemma~\ref{lem:collision} uses only independence and this ball-probability bound, so it applies to the sequence of dimensions with constant $C=1$. The capacity sandwich in Theorem~\ref{thm:capacity} yields $N_\delta^T\sim[-2\log(1-\delta)/q_m(h)]^{1/2}$. Taking logarithms and substituting $h=c/\sqrt m$ gives
\[
 \log N_\delta^T
 =-\tfrac m2\log h+\tfrac14\log m+O(1)
 =\frac m4\log(m/c^2)+O(\log m),
\]
as claimed. For the typical capacity, the same ball-probability identity and~\eqref{eq:typicalhomogeneous} give
\begin{align*}
 \log N_{\mathrm{typ},\delta}^T
 &=-\log q_m(h)+\log[-\log(1-\delta)]+o(1)\\
 &=-m\log h+\tfrac12\log m+O(1)
 =\tfrac m2\log(m/c^2)+O(\log m).
\end{align*}
Condition~\eqref{eq:retrievalradius} requires $\Delta<h$, so the single-kernel retrieval guarantee uses radii at most $O(m^{-1/2})$ along this sequence. No upper bound on every possible attraction basin is asserted.
\end{proof}

\section{Proof of the conditional emergence law}\label{app:emergence}
The edge-count argument is a marked version of the close-pair Poisson method of \citet{silverman}; see also \citet{penrose}. The geometric part below identifies precisely which edges create memories for each energy.

\begin{lemma}[Two disjoint classes of close edges]\label{lem:marked}
In the fixed-manifold continuous-density setting, let $h\to0$ and $N^2q_h\to2\lambda$. Let $I_h$ count pairs at distance less than $h$ and let $J_h$ count pairs at distance in $(ah,2h)$ for a fixed $1\le a<2$. Then, for each integer $j\ge0$,
\[
 \P(I_h=0,J_h=j)\longrightarrow
 e^{-(\lambda+\lambda_J)}\frac{\lambda_J^j}{j!},
 \qquad\lambda_J=(2^m-a^m)\lambda.
\]
In particular, $J_h\mid\{I_h=0\}\Longrightarrow\Poisson(\lambda_J)$.
\end{lemma}
\begin{proof}
The small-ball expansion in Theorem~\ref{thm:capacity} gives
$q_{rh}/q_h\to r^m$ for each fixed $r>0$. The probability of a $J$-edge is $p_J=q_{2h}-q_{ah}$, and the probability of an edge of either counted type is $p_U=q_h+p_J$. Thus
\[
 \binom N2p_J\to\lambda_J,\qquad\binom N2p_U\to\lambda+\lambda_J.
\]
For fixed nonnegative integers $j,k$, put
\[
 T_{j,k}=\E\left[\binom{J_h}{j}\binom{I_h+J_h-j}{k}\right],
\]
with the product defined as zero when $J_h<j$. It selects $j$ distinct $J$-edges and $k$ further distinct edges of either counted type. Disjoint-endpoint selections contribute
\[
 \frac{(N)_{2(j+k)}}{2^{j+k}j!k!}\,p_J^j p_U^k
 \longrightarrow\frac{\lambda_J^j}{j!}\frac{(\lambda+\lambda_J)^k}{k!}.
\]
Any selection with overlapping endpoints has a spanning forest with the notation $v>2c$ from Lemma~\ref{lem:collision}. All its edges have distance less than $2h$. Since $\sup_xb_{2h}(x)=O(q_h)$, its probability is at most $(Cq_h)^{v-c}$ and its total labeled contribution tends to zero exactly as in that lemma. Therefore $T_{j,k}$ has the displayed limit.

Apply the odd/even binomial bounds from Lemma~\ref{lem:collision} to the nonnegative integer $I_h+J_h-j$ on $\{J_h\ge j\}$, multiply them by $\binom{J_h}{j}$, and set the products to zero outside this event. The indicator bounded in this way is
\[
 \binom{J_h}{j}\ind\{I_h+J_h-j=0\}=\ind\{I_h=0,J_h=j\}.
\]
For any fixed truncation order the expected bounds are alternating sums of $T_{j,k}$. First passing to the bandwidth limit and then sending the truncation order to infinity gives
\[
 \P(I_h=0,J_h=j)\to\frac{\lambda_J^j}{j!}
 \sum_{k=0}^\infty\frac{(-1)^k(\lambda+\lambda_J)^k}{k!}
 =e^{-(\lambda+\lambda_J)}\frac{\lambda_J^j}{j!}.
\]
Lemma~\ref{lem:collision} gives $\P(I_h=0)\to e^{-\lambda}>0$. Dividing proves convergence of each conditional point probability to the Poisson point probabilities. Their sum is one, so these pointwise limits imply convergence in distribution: for any integer cutoff, sum finitely many point probabilities, and then let the cutoff increase to make the limiting tail arbitrarily small.
\end{proof}

\begin{proof}[Proof of Theorem~\ref{thm:emergence}, independent-pattern part]
Join two patterns by an edge when their distance is less than $2h$. A connected component with three or more vertices contains two edges sharing a vertex. For any specified ordered triple, the probability of two such edges is at most $(\sup_xb_{2h}(x))^2=O(h^{2m})$. A union bound gives
\[
 \P(\text{some component has at least three vertices})=O(N^3h^{2m}).
\]
Since $q_h\asymp h^m$ and $N^2q_h\to2\lambda$, this is $O(h^{m/2})\to0$. With probability tending to one, the graph is a union of isolated vertices and isolated pairs.

By Lemma~\ref{lem:small} and Theorem~\ref{thm:active}, every nondegenerate memory has an active set of size at most two on this graph event. A singleton candidate is exactly its original. For a pair at separation $r<2h$, the geodesic midpoint $c$ has opposite logarithms to the endpoints, each of length $r/2<h$. It is therefore the unique two-kernel geodesic candidate, with positive-definite energy Hessian.

We verify the corresponding corrected candidate and its displacement uniformly over all such pairs away from a negligible shell. At the midpoint, let $a_1,a_2$ and $q_1,q_2$ be the two components. Their corrected active-score gradient is
\[
 \grad G_{\{1,2\}}(c)
 =\beta\sum_{i=1}^2(a_i(c)-1)\log_cX_i
 +\sum_{i=1}^2(1-\beta q_i(c))\grad a_i(c).
\]
The unweighted logarithms cancel. Lemma~\ref{lem:small} gives $a_i-1=O(h^2)$, $\norm{\log_cX_i}\le h$, and $\norm{\grad a_i}=O(h)$. Because $\beta=2/h^2$, the displayed gradient is $O(h)$. The same lemma and the proof of Theorem~\ref{thm:active} give
\[
 \Hess G_{\{1,2\}}\preceq-c_0h^{-2}g
\]
on the support intersection, for a uniform $c_0>0$.

Choose a fixed $D$ large enough. If $r<2h-2Dh^3$, the closed ball of radius $Dh^3$ about $c$ lies within both supports. Along any unit-speed radial geodesic from $c$, the outward derivative of $G_{\{1,2\}}$ at this ball's boundary is at most
\[
 C_0h-c_0h^{-2}Dh^3=(C_0-c_0D)h<0.
\]
The score attains a maximum on the compact closed ball. Its negative outward derivative excludes the boundary, so it has an interior critical point. Strict concavity makes it unique. Its distance from the midpoint is at most $Dh^3$, and the energy Hessian there is positive definite.

We now bound the exceptional shells. A continuous density is bounded; uniformly in $x$, a shell at radius $bh$, for fixed $b>0$, and thickness $Ch^3$ has probability $O(h^{m-1}h^3)=O(h^{m+2})$. This follows by normal-coordinate polar integration, with the volume density uniformly bounded. Therefore the expected number of sample pairs in such a shell is $O(N^2h^{m+2})=O(h^2)\to0$. In particular, pairs within $2Dh^3$ of radius $2h$ may be discarded with probability tending to one.

The pair candidate has no active third pattern. If a third pattern were within distance $h$ of it, the triangle inequality would put that pattern within distance $2h$ of an endpoint, contradicting the isolated-pair graph event. Such excluded patterns are at distance at least $h$ from the candidate and hence satisfy the novelty condition for every $\tau<1$.

By Theorem~\ref{thm:capacity}, $\mathcal A_T=\{I_h=0\}$ almost surely for sufficiently small $h$. On this event every pair separation exceeds $h$. The geodesic midpoint is at distance exactly $r/2$ from its two endpoints. The corrected candidate has each endpoint distance between $r/2-Dh^3$ and $r/2+Dh^3$. Except for another negligible shell around $2\tau h$ when this radius lies in $[h,2h)$, the candidate is $\tau h$-novel exactly when
$r>\max\{1,2\tau\}h=a_\tau h$. When $2\tau<1$, all allowed pair separations already meet novelty for small $h$.

We have shown that $N_\emg^T(\tau h)$ agrees with the pair count $J_h$ from Lemma~\ref{lem:marked}, on $\mathcal A_T$, with probability tending to one. All excluded events have unconditional probability tending to zero. Since $\P(\mathcal A_T)\to e^{-\lambda}>0$, their conditional probabilities also tend to zero. Applying Lemma~\ref{lem:marked} proves~\eqref{eq:emergencepoisson}. In particular the conditional probability of no novel memory tends to $e^{-\lambda_\tau}$, and multiplication by $\P(\mathcal A_T)$ proves~\eqref{eq:globalprob}.
\end{proof}

Every supported memory is within distance $h$ of at least one original. Hence $N_\emg^T(\eta)=0$ whenever $\eta\ge h$. A fixed positive novelty threshold must not be retained unchanged in a theorem that sends $h\to0$.

\section{Mean-shift identity, descent, and local retrieval rates}\label{app:descent}

\paragraph{The established mean-shift identity.}
For a radial profile $\kappa$, the nonlinear mean-shift displacement of \citet{subbarao}, equations~(29)--(30), is the weighted average of $\log_x\xi_i$ with weights $-\kappa'(d(x,\xi_i)^2/h^2)$, followed by $\exp_x$. With $\kappa(s)=(1-s)_+$, these weights are one for $s<1$ and zero for $s>1$. Their normalized average is exactly $v(x)$ in~\eqref{eq:sparseattention}, and the published update is $T_1(x)$. In the terminology of \citet{cheng,comaniciu}, the Epanechnikov kernel is the \emph{shadow of the flat kernel}; the flat profile is its negative derivative. We therefore describe Algorithm~\ref{alg:radius} as flat-weight Riemannian mean shift. The proof below specializes the classical score-minorization argument to explicit distance-Hessian bounds and memory retrieval. It does not assert global convergence of every nonsmooth trajectory.

\begin{proof}[Proof of Theorem~\ref{prop:descent}]
Set $A=A(x)$, $k=|A|$, $v=v(x)$, and $y=T_t(x)$. On the open active-set region containing $x$, $S_\geo=\epsilon+k-\beta F_A$, with $\grad F_A=-\sum_{i\in A}\log_x\xi_i=-kv$. Hence $\grad S_\geo=\beta kv$, and differentiation of $-\beta^{-1}\log S_\geo$ proves~\eqref{eq:gradientlayer}. Multiplying this gradient by $-tS_\geo/k$ gives $tv$, so its exponential-map gradient step is exactly $T_t$. The singleton and radius assertions were proved in Appendix~\ref{app:retrieval}. Define $\gamma(s)=\exp_x(stv)$, $0\le s\le1$, and $F_A=\sum_{i\in A}q_i$. Its initial directional derivative is
\[
 (F_A\circ\gamma)'(0)=\ip{\grad F_A(x)}{tv}
 =-tk\norm v^2.
\]
The constant-speed geodesic and the assumed Hessian bound give
\[
 (F_A\circ\gamma)''(s)\le k\Lambda t^2\norm v^2.
\]
The identity $F_A(\gamma(1))=F_A(\gamma(0))+(F_A\circ\gamma)'(0)+\int_0^1(1-s)(F_A\circ\gamma)''(s)\,ds$, obtained by integrating twice and changing the order of integration, yields
\[
 F_A(y)\le F_A(x)-tk\norm v^2+\tfrac12k\Lambda t^2\norm v^2
 =F_A(x)-kt(1-\Lambda t/2)\norm v^2.
\]
Since $(1-\beta q_i(y))_+\ge1-\beta q_i(y)$ for every $i\in A$, and all other truncated terms are nonnegative,
\begin{align*}
 S_\geo(y)&\ge\epsilon+\sum_{i\in A}(1-\beta q_i(y))\\
 &=S_\geo(x)+\beta(F_A(x)-F_A(y))\\
 &\ge S_\geo(x)+\beta kt(1-\Lambda t/2)\norm v^2.
\end{align*}
This is strictly larger when $v\ne0$ and $0<t<2/\Lambda$. It also ensures $S_\geo(y)>0$. Strict decrease of the logarithmic energy follows.

For $t=1$, $\norm v\le k^{-1}\sum_{i\in A}d(x,\xi_i)<h$. Every point of the segment has distance less than $2h$ from every active center. On a fixed compact manifold, for sufficiently small $h$ these points lie in the required normal domains and Lemma~\ref{lem:small} applied at radius $2h$ gives $\Lambda=1+O(h^2)<2$. Thus the full step is a descent step unless already critical.

Near a memory $z$ whose active set is $A$, the strict support margins make the active set constant. The update is the smooth map
$T_t(x)=\exp_x(-t\grad F_A(x)/k)$. In normal coordinates at $z$, the map $(x,u)\mapsto\exp_xu$ equals $x$ when $u=0$; its two partial derivatives there are both the identity, since differentiating the constant curve gives the base derivative and differentiating the initial velocity gives the velocity derivative. Since $\grad F_A(z)=0$, the chain rule therefore gives
\[
 DT_t(z)=I-\frac tk\Hess F_A(z).
\]
Every eigenvalue of $\Hess F_A(z)/k$ lies in $[c,\Lambda]$. For $0<t<2/\Lambda$, the convex function $\lambda\mapsto|1-t\lambda|$ on $[c,\Lambda]$ is bounded by its larger endpoint value, so
\begin{equation}
 \norm{DT_t(z)}_{\mathrm{op}}\le r_t:=\max\{|1-tc|,|1-t\Lambda|\}<1.
 \label{eq:localrate}
\end{equation}
The norm identity uses self-adjointness at $z$. Given any $r$ with $r_t<r<1$, continuity gives a sufficiently small normal-coordinate ball on which the derivative norm is at most $r$. The line-segment mean-value bound in those coordinates implies
$\norm{T_t(x)-T_t(z)}\le r\norm{x-z}$, and the ball maps into itself. Iteration therefore gives geometric decay $r^n$ in coordinates. The coordinates are normal and centered at $z$, so the norm of the coordinate vector of any point in this ball equals its geodesic distance from $z$. Thus
\[
 d(T_t^n(x),z)\le r^n d(x,z),\qquad n\ge0,
\]
which proves the local linear rate.
\end{proof}

For a query with nonempty active set $A=A(x)$, put $W(x)=\sum_{i\in A}a_i(x)$ and define the corrected direction
\begin{equation}
 v_\vc(x)=\frac1{W(x)}\left[
 \sum_{i\in A}a_i(x)\log_x\xi_i
 -\frac1\beta\sum_{i\in A}a_i(x)(1-\beta q_i(x))\grad\log\theta_{\xi_i}(x)
 \right].
 \label{eq:correcteddirection}
\end{equation}

\begin{proposition}[Intrinsic equivariance and the corrected direction]\label{prop:layergeometry}
For any Riemannian isometry $\phi:\M\to\M$, applying Algorithm~\ref{alg:radius} to $\phi(x),\phi(\xi_i)$ gives $\phi(T_t(x))$, with the same active flag. Away from support boundaries, the corrected direction in~\eqref{eq:correcteddirection} satisfies $\grad E_\vc=-Wv_\vc/S_\vc$. For a fixed active set and fixed $t$, the geodesic layer is smooth in the query and active keys and is constant with respect to changes in $h$ that preserve that active set.
\end{proposition}
\begin{proof}
An isometry preserves lengths of curves and therefore distances, so it preserves the active indices. It maps geodesics to geodesics with the same initial speeds. Uniqueness of the geodesic within each support gives
\[
 \log_{\phi(x)}\phi(\xi_i)=D\phi_x\log_x\xi_i,
 \qquad
 \phi(\exp_xu)=\exp_{\phi(x)}(D\phi_xu).
\]
The differential $D\phi_x$ is linear. Averaging the first identity over the active indices and using the second gives
\[
 \exp_{\phi(x)}\!\left(\frac t{|A|}\sum_{i\in A}\log_{\phi(x)}\phi(\xi_i)\right)
 =\phi\!\left(\exp_x\!\left(\frac t{|A|}\sum_{i\in A}\log_x\xi_i\right)\right).
\]
For an empty active set both algorithms return their input, so equivariance holds there as well.

For the corrected direction, the product rule and $\grad q_i=-\log_x\xi_i$ give
\begin{align*}
 \grad S_\vc
 &=\sum_{i\in A}\left[(1-\beta q_i)\grad a_i-\beta a_i\grad q_i\right]\\
 &=\beta\sum_{i\in A}a_i\log_x\xi_i
 -\sum_{i\in A}a_i(1-\beta q_i)\grad\log\theta_{\xi_i}
 =\beta Wv_\vc.
\end{align*}
Substitution into~\eqref{eq:loghessian}'s gradient identity proves $\grad E_\vc=-Wv_\vc/S_\vc$.

Finally, strict support inequalities persist under sufficiently small changes in the query, keys, and radius, because distance is continuous and there are finitely many inequalities. Each active logarithm is smooth below the injectivity radius, and the exponential map is smooth. Their finite average and composition are smooth in the query and active keys. Once the active indices and $t$ are fixed, formula~\eqref{eq:sparseattention} contains no $h$. Its derivative with respect to $h$ is therefore zero on that neighborhood. This argument makes no differentiability assertion on support boundaries.
\end{proof}

\section{Robust capacity and an unbounded exact-storage example}\label{app:packing}

On a fixed compact manifold with continuous sampling density $f$, Theorem~\ref{thm:capacity} gives random all-pattern capacity of order $\beta^{m/4}$, while Proposition~\ref{prop:typical} gives typical-pattern capacity of order $\beta^{m/2}$ for a fixed target success probability. Designed storage with a prescribed radius $\vartheta h$, $0<\vartheta<1$, is a different task, governed by the packing bounds below. Exact local minima without a prescribed radius or random-sampling requirement need not have finite capacity.

Let $P_>(s)$ and $P_{\ge}(s)$ be the largest cardinalities of sets with pairwise distances respectively greater than $s$ and at least $s$. Fix a retrieval radius $0<\Delta<h$. Define $C_\Delta^T$ using the negative-energy-gradient flow: every radius-$\Delta$ ball about a pattern must consist of initial states whose trajectories converge to that pattern. All statements below refer to regions where this flow is smooth and uniquely defined.

\begin{proposition}[Packing bounds for robust retrieval]\label{prop:packing}
For the geodesic energy, and for the corrected energy under the uniform form of Assumption~\ref{ass:convex},
\[
 P_>(h+\Delta)\le C_\Delta^T\le P_{\ge}(2\Delta).
\]
For the following volume bounds assume also $2h<\inj(\M)$. Let $V=\vol(\M)$, let $v_+(r)=\sup_x\vol(B_r(x))$, and let $v_-(r)=\inf_x\vol(B_r(x))$. Then
\[
 \frac{V}{v_+(h+\Delta)}\le C_\Delta^T\le\frac{V}{v_-(\Delta)}.
\]
For fixed dimension and a fixed manifold, at $\Delta=\vartheta h$, $0<\vartheta<1$, sufficiently small $h$ gives
$C_\Delta^T=\Theta(h^{-m})=\Theta(\beta^{m/2})$.
\end{proposition}
\begin{proof}
A pattern set separated by more than $h+\Delta$ has only its own kernel active throughout each closed radius-$\Delta$ ball. For $E_\geo$, the negative gradient points radially toward the center by~\eqref{eq:singlegradient}. If $r(t)$ is its distance to the center, then
\[
 \dot r(t)=-\frac{r(t)}{\epsilon+1-\beta r(t)^2/2}
 \le-\frac{r(t)}{1+\epsilon}.
\]
It remains in the ball and converges to the center.

For the corrected energy write $k_i=a_i(1-\beta q_i)$. Its gradient vanishes at the center and $\Hess k_i\preceq-\mu g$ throughout the ball. Along a unit-speed radial geodesic from the center,
\[
 \frac d{dr}k_i(\gamma(r))
 =\int_0^r\Hess k_i(\dot\gamma(s),\dot\gamma(s))\,ds\le-\mu r.
\]
Negative-energy-gradient flow equals $\grad k_i/[\beta(\epsilon+k_i)]$. If $S_{\max}$ bounds $\epsilon+k_i$ on the closed ball, its radial derivative obeys
\[
 \dot r(t)\le-\frac{\mu r(t)}{\beta S_{\max}}.
\]
The vector field points inward on the boundary; the denominator is bounded away from zero since $\Delta<h$. The trajectory therefore exists for all nonnegative time in the ball and converges exponentially to the center. This proves the lower packing bound.

For the upper bound, two prescribed attraction balls cannot overlap: a unique trajectory from a common initial point cannot converge to two distinct centers. If two centers had distance less than $2\Delta$, a midpoint of a minimizing geodesic would lie in both balls. Thus their separation is at least $2\Delta$.

A maximal set separated by more than $s$ covers $\M$ by closed radius-$s$ balls, since otherwise another center could be added. Ball boundaries have zero volume for sufficiently small radii, which are the radii used here; hence its cardinality is at least $V/v_+(s)$. The open radius-$s/2$ balls about an $s$-separated set are disjoint, implying cardinality at most $V/v_-(s/2)$. These observations prove the volume bounds in the small-radius regime. Finally, normal coordinates give $v_\pm(r)=\omega_mr^m(1+O(r^2))$ uniformly on a fixed compact manifold. Substitution of $\Delta=\vartheta h$ proves the order statements.
\end{proof}

\begin{proposition}[Exact storage alone need not have finite capacity]\label{prop:circle}
There is a fixed circle and a fixed bandwidth at which, for arbitrarily large $N$, both energies store $N$ original patterns and have exactly $N$ additional nondegenerate memories. The basin and novelty radii shrink with $N$.
\end{proposition}
\begin{proof}
Take a circle of circumference $L$, fix $0<h<L/4$ with $h/L$ irrational, and put $N$ equally spaced patterns at spacing $\ell=L/N$. Choose $N$ sufficiently large that $\ell<2h$. A one-dimensional Riemannian circle has local normal-coordinate density one, so the energies coincide.

At every original pattern, reflection symmetry pairs active neighbors at opposite signed distances, so the first derivative of the score is zero. There is no pattern exactly on a support boundary because $h/\ell=Nh/L$ is irrational. On a neighborhood with $k$ active patterns the second derivative of the score is $-\beta k<0$. Thus every original is a nondegenerate memory.

The same reflection symmetry holds at each of the $N$ midpoints between consecutive patterns. Each has active neighbors because $\ell/2<h$. No support boundary occurs there because an integer plus one half cannot equal the irrational $h/\ell$. These $N$ midpoints are also nondegenerate memories, each at novelty distance $\ell/2$.

There are exactly $2N$ distinct support-boundary points $\xi_i\pm h$ on the circle. Indeed, equality between a plus and a minus boundary would make $2h/L$ rational; equality within either sign would identify distinct equally spaced centers. Their complement consists of $2N$ intervals. Supports cover the circle because $\ell/2<h$, so each interval has at least one active pattern and strictly negative score second derivative. It contains at most one stationary maximum. Lemma~\ref{lem:boundary} excludes boundary maxima. We have already exhibited $2N$ distinct maxima, so these are all of them. The construction works for every sufficiently large $N$ at the same $L$ and $h$.
\end{proof}

\section{Statistical incompatibility and an alternative normalization}\label{app:normalization}

\begin{corollary}[Incompatible fixed-dimensional asymptotic requirements]\label{cor:incompatible}
Fix a compact manifold and a $C^4$ density, and let $N\to\infty$, $h\to0$. For either energy,
\[
 \P(\mathcal A_T)\to1\quad\Longleftrightarrow\quad N^2h^m\to0.
\]
Mean integrated squared-error consistency of the normalized density $\widetilde f_T$ requires and, with $h\to0$, is implied by $Nh^m\to\infty$. These requirements cannot hold at the same bandwidth. In the high-probability exact-storage regime, the probability of any additional memory also tends to zero.
\end{corollary}

\begin{proof}[Proof of Corollary~\ref{cor:incompatible}]
Theorem~\ref{thm:capacity} identifies all-pattern storage with no pair at distance less than $h$, and $q_h\asymp h^m$. If $N^2q_h\to0$, the union bound gives success probability tending to one. Conversely, suppose success probability tends to one but $N^2q_h$ does not tend to zero. There is a subsequence along which $N^2q_h$ is bounded below by a positive constant. It has a further subsequence along which this quantity either converges to a finite positive limit or tends to infinity. In the first case Lemma~\ref{lem:collision} gives a success limit strictly below one; in the second it gives limit zero. Both contradict the assumed success. Thus $N^2q_h\to0$, proving the storage equivalence.

For either normalized density, the risk expansion~\eqref{eq:mise} in Theorem~\ref{thm:kde} has a nonnegative squared-bias leading term and a variance leading term $R_m/(Nh^m)$ with $R_m>0$. If $Nh^m$ failed to tend to infinity, a subsequence would have $(Nh^m)^{-1}$ bounded below. Along that subsequence $h^4=o((Nh^m)^{-1})$, so the remainder is $o((Nh^m)^{-1})$ and the risk is bounded below by $R_m/(2Nh^m)$ for small $h$. This contradicts consistency. Hence $Nh^m\to\infty$. Conversely, if $h\to0$ and $Nh^m\to\infty$, both terms in~\eqref{eq:mise} tend to zero. But $N^2h^m\to0$ implies $Nh^m\to0$, not infinity. The two requirements are incompatible.

Finally, in the high-probability storage regime,
$\binom N2q_{2h}\to0$ because $q_{2h}/q_h\to2^m$. A union bound shows that all support balls are disjoint with probability tending to one. For small $h$, a single active component has its unique critical maximum at its original center by Theorem~\ref{thm:active}. No additional supported memories are then possible.
\end{proof}

\begin{proposition}[Normalization by integrated kernel mass]\label{prop:massnorm}
Let
\[
 Z_h(p)=\int_\M(1-d(x,p)^2/h^2)_+\,dV(x),\qquad w_i=Z_h(\xi_i)^{-1}.
\]
Then $N^{-1}\sum_iw_i(1-d(x,\xi_i)^2/h^2)_+$ is a normalized density. The energy
\[
 E_{\mathrm{mass}}(x)=-\beta^{-1}\log\left[\epsilon+\sum_iw_i(1-\beta q_i(x))_+\right]
\]
retains every isolated original as a nondegenerate minimum, with
\[
 \Hess E_{\mathrm{mass}}(\xi_i)=\frac{w_i}{\epsilon+w_i}g_{\xi_i}.
\]
\end{proposition}
\begin{proof}
The integral $Z_h(p)$ is positive and finite because its integrand is nonnegative, is positive on $B_h(p)$, and the manifold is compact. Dividing the component by its integral gives mass one; averaging gives a normalized density. At an isolated original the score is locally $\epsilon+w_i(1-\beta q_i)$. Its differential is zero and its Hessian is $-\beta w_ig$. Substitution into~\eqref{eq:loghessian} proves the asserted positive-definite energy Hessian.
\end{proof}

This integrated-mass construction normalizes each kernel by an $x$-independent factor. Its stability differs from pointwise volume correction because that factor contributes no spatial gradient or Hessian.

 \end{document}